\documentclass[11pt]{article}
\PassOptionsToPackage{numbers, compress, sort}{natbib}

\usepackage{microtype}
\usepackage{graphicx}
\usepackage{subfigure}
\usepackage{booktabs} 
\usepackage{xcolor}
\usepackage{amsthm}
\usepackage{amssymb}
\usepackage{amsmath}
\usepackage{mathtools}
\usepackage{verbatim}
\usepackage{natbib} 
\usepackage{enumitem}
\usepackage{tikz}
\usepackage{authblk}

\usepackage{algorithm}
\usepackage{algorithmic}
\usepackage{float} 

\DeclareMathOperator{\E}{\mathbb{E}}

\DeclareMathOperator*{\argmin}{arg\,min}

\DeclareMathOperator{\point}{\mathnormal{v}}
\DeclareMathOperator{\Points}{\mathcal{C}}
\DeclareMathOperator{\pcolor}{\mathnormal{h_{\ell}}}
\DeclareMathOperator{\Colors}{\mathcal{H}}
\DeclareMathOperator{\Income}{\mathcal{H}_{\mathnormal{R}}}

\DeclareMathOperator{\distt}{\mathcal{D}_{\mathrm{True}}}
\DeclareMathOperator{\disti}{\mathcal{D}_{\mathrm{Indep}}}
\DeclareMathOperator{\pacc}{\mathnormal{p}_{\text{acc}}}
\DeclareMathOperator{\pn}{\mathnormal{p}_{\text{noise}}}

\newcommand{\adult}{\textbf{Adult}}
\newcommand{\bank}{\textbf{Bank}}
\newcommand{\credit}{\textbf{CreditCard}}
\newcommand{\cens}{\textbf{Census1990}}

\DeclareMathOperator{\POF}{\mathrm{POF}}

\DeclareMathOperator{\OPT}{\mathrm{OPT}}
\DeclareMathOperator{\PFC}{\mathrm{PFC}}

\DeclareMathOperator{\DFCLB}{\mathrm{DFC}_{\mathrm{LB}}}
\DeclareMathOperator{\instpfc}{\mathcal{I}_{\mathrm{PFC}}}

\DeclareMathOperator{\FAIRAPFC}{\mathrm{FA-PFC}}
\DeclareMathOperator{\FAIRAPFCLB}{\mathrm{FA-PFC-LB}}
\DeclareMathOperator{\OPTPFC}{\mathrm{OPT}_{\mathrm{PFC}}}
\DeclareMathOperator{\clust}{\mathrm{Cluster}}

\newcommand{\floor}[1]{\left\lfloor #1 \right\rfloor}
\newcommand{\ceil}[1]{\left\lceil #1 \right\rceil}
\newtheorem{theorem}{Theorem}[section]

\newtheorem{assumption}{Assumption}[section]

\newtheorem{lemma}{Lemma}[section]

\DeclarePairedDelimiter\abs{\lvert}{\rvert}

\begin{document}
%

\title{Probabilistic Fair Clustering\footnote{An earlier version of this paper was published in NeurIPS 2020. This version is updated to include a correction to the solution for the multi-color case under the large cluster assumption from polynomial time to fixed-parameter tractable.}}
\author[]{Seyed A. Esmaeili\thanks{esmaeili@cs.umd.edu}}
\author[]{Brian Brubach\thanks{bbrubach@cs.umd.edu}}
\author[]{Leonidas Tsepenekas\thanks{ltsepene@cs.umd.edu}}
\author[]{John P. Dickerson\thanks{john@cs.umd.edu}}
\affil[]{Department of Computer Science, University of Maryland, College Park}

\date{}

\maketitle              

\begin{abstract}
  In clustering problems, a central decision-maker is given a complete metric graph over vertices and must provide a clustering of vertices that minimizes some objective function.  In \emph{fair} clustering problems, vertices are endowed with a \emph{color} (e.g., membership in a group), and the features of a valid clustering might also include the representation of colors in that clustering.  Prior work in fair clustering assumes complete knowledge of group membership. In this paper, we generalize prior work by assuming imperfect knowledge of group membership through probabilistic assignments. We present clustering algorithms in this more general setting with approximation ratio guarantees.  We also address the problem of ``metric membership,'' where different groups have a notion of order and distance. Experiments are conducted using our proposed algorithms as well as baselines to validate our approach and also surface nuanced concerns when group membership is not known deterministically.
\end{abstract}

\section{Introduction}\label{sec:intro}
Machine-learning-based decisioning systems are increasingly used in high-stakes situations, many of which directly or indirectly impact society.  Examples abound of automated decisioning systems resulting in, arguably, morally repugnant outcomes: hiring algorithms may encode the biases of human reviewers’ training data~\cite{BoRi18a}, advertising systems may discriminate based on race and inferred gender in harmful ways~\cite{sweeney2013discrimination}, recidivism risk assessment software may bias its risk assessment improperly by race~\cite{angwin2016machine}, and healthcare resource allocation systems may be biased against a specific race~\cite{ledford2019millions}.  A myriad of examples such as these and others motivate the growing body of research into defining, measuring, and (partially) mitigating concerns of fairness and bias in machine learning.  Different metrics of algorithmic fairness have been proposed, drawing on prior legal rulings and philosophical concepts;~\cite{mehrabi2019survey} give a recent overview of sources of bias and fairness as presently defined by the machine learning community.

The earliest work in this space focused on fairness in \emph{supervised} learning~\cite{Luong11:KNN,Hardt16:Equality} as well as \emph{online} learning~\cite{Joseph16:Fairness}; more recently, the literature has begun expanding into fairness in \emph{unsupervised} learning~\cite{chierichetti2017fair}.  In this work, we address a novel model of fairness in clustering---a fundamental unsupervised learning problem.  Here, we are given a complete metric graph where each vertex also has a color associated with it, and we are concerned with finding a clustering that takes both the metric graph and vertex colors into account. Most of the work in this area~(e.g.,~\cite{ahmadian2019clustering, bercea2018cost, chierichetti2017fair}) has defined a fair clustering to be one that minimizes the cost function subject to the constraint that each cluster satisfies a lower and an upper bound on the percentage of each color it contains---a form of approximate \emph{demographic parity} or its closely-related cousin, the $p\%$-rule~\cite{biddle2006adverse}.  We relax the assumption that a vertex’s color assignment is known deterministically; rather, for each vertex, we assume only knowledge of a distribution over colors.

Our proposed model addresses many real-world use cases.  \cite{ahmadian2019clustering} discuss clustering news articles such that no political viewpoint---assumed to be known deterministically---dominates any cluster.  Here, the color membership attribute---i.e., the political viewpoint espoused by a news article---would not be provided directly but could be predicted with some probability of error using other available features.  \cite{awasthi2019effectiveness} discuss the case of supervised learning when class labels are not known with certainty (e.g., due to noisy crowdsourcing or the use of a predictive model).  Our model addresses such motivating applications in the unsupervised learning setting, by defining a fair cluster to be one where the color proportions satisfy the upper and lower bound constraints \emph{in expectation}. Hence, it captures standard deterministic fair clustering as a special case.

\noindent\textbf{Outline \& Contributions.}  We begin (\S\ref{sec:rw}) with an overview of related research in general clustering, fairness in general machine learning, as well as recent work addressing fairness in unsupervised learning.  Next (\S\ref{sec:prelims}), we define two novel models of clustering when only probabilistic membership is available: the first assumes that colors are unordered, and the second embeds colors into a metric space, thus endowing them with a notion of order and distance.  This latter setting addresses use cases where, e.g., we may want to cluster according to membership in classes such as age or income, whose values are naturally ordered.  Following this (\S\ref{theory_sec}), we present two approximation algorithms with theoretical guarantees in the settings above.  We also briefly address the (easier but often realistic) ``large cluster'' setting, where it is assumed that the optimal solution does not contain pathologically small clusters.  Finally (\S\ref{sec:experiments}), we verify our proposed approaches on four real-world datasets. We note that all proofs are put in the appendix due to the page limit. 

\section{Related Work}\label{sec:rw}
Classical forms of the metric clustering problems $k$-center, $k$-median, and $k$-means are well-studied within the context of unsupervised learning and operations research. While all of these problems are NP-hard, there is a long line of work on approximating them and heuristics are commonly used in many practical applications. This vast area is surveyed by~ \cite{aggarwal2013data} and we focus on approximation algorithms here. For $k$-center, there are multiple approaches to achieve a $2$-approximation and this is the best possible unless $P = NP$~\cite{Hochbaum1985,Gonzalez1985,Hochbaum1986}. Searches for the best approximations to $k$-median and $k$-means are ongoing. For $k$-median there is a $(2.675 + \epsilon)$-approximation with a running time of $n^{O((1/\epsilon)\log(1/\epsilon))}$~\cite{byrka2014improved} and for $k$-means, a $6.357$-approximation is the best known~\cite{ahmadian2019better}.

The study of approximation algorithms that achieve demographic fairness for metric clustering was initiated by~\cite{chierichetti2017fair}. They considered a variant of $k$-center and $k$-median wherein each point is assigned one of two colors and the color of each point is known. Followup work extended the problem setting to the $k$-means objective, multiple colors, and the possibility of a point being assigned multiple colors (i.e. modeling intersecting demographic groups such as gender and race combined)~\cite{bercea2018cost,bera2019fair,backurs2019scalable,NIPS2019_8976}. Other work considers the one-sided problem of preventing over-representation of any one group in each cluster rather than strictly enforcing that clusters maintain proportionality of all groups~\cite{ahmadian2019clustering}.

In all of the aforementioned cases, the colors (demographic groups) assigned to the points are known a priori. By contrast, we consider a generalization where points are assigned a distribution on colors. We note that this generalizes settings where each point is assigned a single deterministic color. By contrast, our setting is distinct from the setting where points are assigned multiple colors in that we assume each point has a single true color. In the area of supervised learning, the work of~\cite{awasthi2019effectiveness} addressed a similar model of uncertain group membership. Other recent work explores unobserved protected classes from the perspective of assessment~\cite{kallus2019assessing}. However, no prior work has addressed this model of uncertainty for metric clustering problems in unsupervised learning.

\section{Preliminaries and Problem Definition}\label{sec:prelims}
Let $\Points$ be the set of points in a metric space with distance function $d:\Points \times \Points \rightarrow \mathbf{R}_{\ge 0}$. The distance between a point $\point$ and a set $S$ is defined as $d(\point,S)=\min_{j \in S} d(i,j)$. In a $k$-clustering an objective function $L^k({\Points})$ is given, a set $S \subseteq \Points$ of at most $k$ points must be chosen as the set of centers, and each point in $\Points$ must get assigned to a center in $S$ through an assignment function $\phi:\Points \rightarrow S$ forming a $k$-partition of the original set: $\Points_1,\dots,\Points_k$. The optimal solution is defined as a set of centers and an assignment function that minimizes the objective $L^k({\Points})$. The well known $k$-center, $k$-median, and $k$-means can all be stated as the following problem:
\begin{equation}\label{eq:k_cluster}
    \min_{S: |S| \leq k,\phi} L^k_{p}(\Points) = \min_{S: |S| \leq k,\phi}  \Big(\sum_{\point \in \Points} d^p(\point,\phi(\point))\Big)^{1/p}
\end{equation}
where $p$ equals $\infty,1,$ and $2$ for the case of the $k$-center, $k$-median, and $k$-means, respectively. For such problems the optimal assignment for a point $\point$ is the nearest point in the chosen set of centers $S$. However, in the presence of additional constraints such as imposing a lower bound on the cluster size \cite{aggarwal2010achieving} or an upper bound \cite{khuller2000capacitated,cygan2012lp,an2015centrality} this property no longer holds. This is also true for fair clustering. 

To formulate the fair clustering problem, a set of colors $\Colors=\{h_1,\dots,\pcolor,\dots,h_{m}\}$ is introduced and each point $\point$ is mapped to a color through a given function $\chi: \Points \rightarrow \Colors$. Previous work in fair clustering \cite{chierichetti2017fair,ahmadian2019clustering,bercea2018cost,bera2019fair} adds to the objective function of (\ref{eq:k_cluster}) the following proportional representation constraint, i.e.:
\begin{equation} \label{eq:fair_constraint}
    \forall i \in S, \forall \pcolor \in \Colors : l_{\pcolor} |\Points_i| \leq |\Points_{i,{\pcolor}}| \leq u_{\pcolor} |\Points_i|
\end{equation}
where $\Points_{i,{h_{\ell}}}$ is the set of points in cluster $i$ having color $h_{\ell}$.  The bounds $l_{\pcolor} , u_{\pcolor} \in (0,1)$ are given lower and upper bounds on the proportion of a given color in each cluster, respectively. 

In this work we generalize the problem by assuming that the color of each point is not known deterministically but rather probabilistically. We also address the case where the colors lie in a 1-dimensional Euclidean metric space.

\subsection{Probabilistic Fair Clustering}\label{pfc}
In \emph{probabilistic fair clustering}, we generalize the problem by assuming that the color of each point is not known deterministically but rather probabilistically. That is, each point $\point$ has a given value $p^{\pcolor}_{\point}$ for each $\pcolor \in \Colors$, representing the probability that point $\point$ has color $\pcolor$, with $\sum_{\pcolor \in \Colors}p^{\pcolor}_{\point} = 1$. 

The constraints are then modified to have the expected color of each cluster fall within the given lower and upper bounds. This leads to the following optimization problem:
\begin{subequations}\label{pfc_opt}
  \begin{equation}
    \label{pfc_eq_1}
     \min_{S: |S| \leq k,\phi} L^k_p(\Points) \\
  \end{equation}
  \begin{equation}
    \label{pfc_eq_2}
     \text{s.t. } \forall i \in S, \forall \pcolor \in \Colors: l_{h_{\ell}} |\phi^{-1}(i)| \leq \sum_{\point \in \phi^{-1}(i)} p^{h_{\ell}}_{\point} \leq u_{h_{\ell}} |\phi^{-1}(i)|
  \end{equation}
\end{subequations}
where $\phi^{-1}(i)$ refers to the set of points assigned to cluster $i$, or in other words $\Points_i$.  

Following~\cite{bera2019fair}, we define a $\gamma$ violating solution to be one for which for all $i \in S$: 
\begin{align}\label{viol_pfc}
    l_{h_{\ell}} |\phi^{-1}(i)|-\gamma \leq \sum_{\point \in \phi^{-1}(i)} p^{h_{\ell}}_{\point} \leq u_{h_{\ell}} |\phi^{-1}(i)|+\gamma
\end{align}
This notion effectively captures the amount $\gamma$, by which a certain solution violates the fairness constraints.

\subsection{Metric Membership Fair Clustering}\label{mbfc}
Representing a point's (individual's) membership using colors may be sufficient for binary or other unordered categorical variables. However, this may leave information ``on the table'' when a category is, for example, income or age, since colors do not have an inherent sense of order or distance.

For this type of attribute, the membership can be characterized by a 1-dimensional Euclidean space. Without loss of generality, we can represent the set of all possible memberships as the set of all consecutive integers from $0$ to some $R>0$, where $R$ is the maximum value that can be encountered. Hence, let $\Income=\{0,\dots,r,\dots,R\}$ where $r$ is an integer and $r\geq 0$. Each point $\point$ has associated with it a value $r_{\point} \in \Income$. In this problem we require the average total value of each cluster to be within a given interval. Hence: 
\begin{subequations}
  \begin{equation}
    \label{mm_eq_1}
     \min_{S: |S|\leq k,\phi} L^k_p(\Points) \\
  \end{equation}
  \begin{equation}
    \label{mm_eq_2}
    \text{s.t. } \forall i \in S: l |\phi^{-1}(i)| \leq \sum_{\point \in \phi^{-1}(i)} r_{\point} \leq u |\phi^{-1}(i)| 
  \end{equation}
\end{subequations}
where $l$ and $u$ are respectively upper and lower bounds imposed on each cluster. 

Similar to section \ref{pfc}, we define a $\gamma$ violating solution to be one for which $\forall i \in S$: 
\begin{align}\label{viol_mm}
    l |\phi^{-1}(i)|-\gamma \leq \sum_{\point \in \phi^{-1}(i)} r_{\point}  \leq u |\phi^{-1}(i)|+\gamma
\end{align}

\section{Approximation Algorithms and Theoretical Guarantees} \label{theory_sec}
\subsection{Algorithms for the Two Color and Metric Membership Case} \label{theory_two_color}
Our algorithm follows the two step method of \cite{bera2019fair}, although we differ in the LP rounding scheme.  Let $\PFC(k,p)$ denote the probabilistic fair clustering problem. The color-blind clustering problem, where we drop the fairness constraints, is denoted by $\clust(k,p)$. Further, define the fair assignment problem $\FAIRAPFC(S,p)$ as the problem where we are given a fixed set of centers $S$ and the objective is to find an assignment $\phi$ minimizing $L^k_p(\Points)$ and satisfying the fairness constraints \ref{pfc_eq_2} for probabilistic fair clustering or \ref{mm_eq_2} for metric-membership.
We prove the following (similar to theorem 2 in \cite{bera2019fair}):
\begin{theorem}\label{bera_theorem}
Given an $\alpha$-approximation algorithm for $\mathrm{Cluster}(k,p)$ and a $\gamma$-violating algorithm for $\FAIRAPFC(S,p)$, a solution with approximation ratio $\alpha+2$ and constraint violation at most $\gamma$ can be achieved for $\PFC(k,p)$. 
\end{theorem}
\begin{proof}
See appendix \ref{p1}
\end{proof}

An identical theorem and proof follows for the metric membership problem as well. 
\subsubsection{Step 1, Color-Blind Approximation Algorithm:}\label{step1}
At this step an ordinary (color-blind) $\alpha$-approximation algorithm is used to find the cluster centers. For example, the Gonzalez algorithm \cite{gonzalez1985clustering} can be used for the $k$-center problem or the algorithm of \cite{byrka2014improved} can be used for the $k$-median. This step results in a set $S$ of cluster centers.
Since this step does not take fairness into account, the resulting solution does not necessarily satisfy constraints \ref{pfc_eq_2} for probabilistic fair clustering and \ref{mm_eq_2} for metric-membership.

\subsubsection{Step 2, Fair Assignment Problem:}\label{step2}
In this step, a linear program (LP) is set up to satisfy the fairness constraints. The variables of the LP are $x_{ij}$ denoting the assignment of point $j$ to center $i$ in $S$. Specifically, the LP is:



\begin{subequations}
 \begin{equation}
    \label{LP1}
    \min \sum_{j \in \Points, i \in S} d^p(i,j) x_{ij} \\
 \end{equation}
 \begin{equation}
    \label{LP0}
    \text{s.t. } \forall i \in S \text{ and } \forall \pcolor \in \Colors: 
 \end{equation}
 \begin{equation}
    \label{LPf1}
    l_{\pcolor} \sum_{j \in \Points} x_{ij} \leq \sum_{j \in \Points} p^{\pcolor}_{j}x_{ij} \leq u_{\pcolor} \sum_{j \in \Points}x_{ij}
 \end{equation}
 \begin{equation}
    \label{LP_one}
    \forall j \in \Points: \sum_{i \in S} x_{ij} =1 , \quad 0 \leq x_{ij} \leq 1
 \end{equation}
\end{subequations}
Since the LP above is a relaxation of $\FAIRAPFC(S,p)$, we have $\OPT^{\mathrm{LP}}_{\FAIRAPFC} \leq \OPT_{\FAIRAPFC}$. We note that for $k$-center there is no objective, instead we have the following additional constraint: $x_{ij}=0 \text{ if } d(i,j) >w$ where $w$ is a guess of the optimal radius. Also, for $k$-center the optimal value is always the distance between two points. Hence, through a binary search over the polynomially-sized set of distance choices we can WLOG obtain the minimum satisfying distance. Further, for the metric membership case $p^{\pcolor}_{j}, l_{\pcolor}$ and $u_{j}$  in \ref{LPf1} are replaced by $r_{j}, l$ and $u$, respectively.

What remains is to round the fractional assignments $x_{ij}$ resulting from solving the LP. 

\subsubsection{Rounding for the Two Color and Metric Membership Case}\label{two_color_rounding}
First we note the connection between the metric membership problem and the two color case of probabilistic fair clustering. Effectively the set $\Income=\{0,1, \dots,R\}$ is the unnormalized version of the set of probabilities $\{0,\frac{1}{R},\frac{2}{R},\dots,1\}$. 
\begin{algorithm}[h!]
   \caption{Form Flow Network Edges for Culster $C_i$}
   \label{alg:nf_const_alg}
\begin{algorithmic}
   \STATE $\vec{A}_{i}$ are the points $j \in \phi^{-1}(i)$ in non-increasing order of $p_j$
   \STATE  initialize array $\vec{a}$ of size $\lvert C_i\rvert$ to zeros, and set $s=1$
   \STATE  put the assignment $x_{ij}$ for each point $j$ in $\vec{A}_{i}$ in $\vec{z}_{i}$ according the vertex order in $\vec{A}_{i}$
   \FOR{$q=1$ {\bfseries to} $\lvert C_i\rvert$}
        \STATE $\vec{a}(q) = \vec{a}(q) + x_{i\vec{A}_{i}(s)}$, and add edge $(\vec{A}_{i}(s),q)$
        \STATE $\vec{z}_{i}(s)=0$
        \STATE $s=s+1$ \COMMENT{Move to the next vertex} 
        \REPEAT  
        \STATE $\mathrm{valueToAdd} = min(1-\vec{a}(q),\vec{z}_{i}(s))$
    	\STATE $\vec{a}(q) = \vec{a}(q) + \mathrm{valueToAdd}$, and add edge $(\vec{A}_{i}(s),q)$
    	\STATE $\vec{z}_{i}(s) = \vec{z}_{i}(s) - \mathrm{valueToAdd}$
    	    \IF{$\vec{z}_{i}(s)=0$}
    	        \STATE $s=s+1$
	        \ENDIF 
        \UNTIL $\vec{a}(q)=1$ or $s>\lvert \vec{A}_i\rvert$ \COMMENT{until we have accumulated 1 or ran out of vertices}
   \ENDFOR
\end{algorithmic}
\end{algorithm}
Our rounding method is based on calculating a minimum-cost flow in a carefully constructed graph. For each $i \in S$, a set $C_{i}$ with $|C_{i}|=\ceil{\sum_{j \in \Points} x_{ij}}$ vertices is created. Moreover, the set of vertices assigned to cluster $i$, i.e. $\phi^{-1}(i)=\{j \in \Points|x_{ij}>0\}$ are sorted in a non-increasing order according to the associated value $r_j$ and placed into the array $\vec{A}_{i}$. A vertex in $C_i$ (except possibly the last) is connected to as many vertices in $\vec{A}_{i}$ by their sorting order until it accumulates an assignment value of 1. A vertex in $\vec{A}_{i}$ may be connected to more than one vertex in $C_i$ if that causes the first vertex in $C_i$ to accumulate an assignment value of 1 with some assignment still remaining in the $\vec{A}_{i}$ vertex. In this case the second vertex in $C_i$ would take only what remains of the assignment. See Algorithm \ref{alg:nf_const_alg} for full details. Appendix \ref{NF_details} demonstrates an example. 

We denote the set of edges that connect all points in $\Points$ to points in $C_i$ by $E_{\Points,C_i}$. Also, let $V_{\text{flow}} = \Points \cup (\cup_{i \in S} C_i) \cup S \cup \{t\}$ and $E_{\text{flow}} = E_{\Points,C_i} \cup E_{C_i,S} \cup E_{S,t}$, where $E_{C_i,S}$ has an edge from every vertex in $C_i$ to the corresponding center $i \in S$. Finally $E_{S,t}$ has an edge from every vertex $i$ in $S$ to the sink $t$ if $\sum_{j \in \Points} x_{ij}>\floor{\sum_{j\in \Points} x_{ij}}$. The demands, capacities, and costs of the network are:
\begin{itemize}
    \item \textbf{Demands}: Each $v \in \Points$ has demand $d_v= -1$ (a supply of 1), $d_u= 0$ for each $u \in C_i$, $d_i=\floor{\sum_{j\in \Points} x_{ij}}$ for each $i \in S$. Finally $t$ has demand $d_t=\lvert\Points\rvert-\sum_{i \in S} d_i$. 
    \item \textbf{Capacities}: All edge capacities are set to 1.
    \item \textbf{Costs}: All edges have cost 0, expect the edges in $E_{\Points,C_i}$ where $\forall (v,u) \in E_{\Points,C_i}, d(v,u)=d(v,i) \text{ for the $k$-median and } d(v,u)=d^2(v,i)$. For the $k$-center, either setting suffices. 
\end{itemize}

\begin{figure}
\centering
\begin{tikzpicture}
[
    xscale=0.35, yscale=0.35,auto,thick,
    st node/.style={
        inner sep=0pt,minimum size=12pt,
        circle, fill=white, draw
  	},
  	v node/.style={
        inner sep=0pt,minimum size=12pt,
        circle, fill=white, draw, font=\small
  	},
  	c node/.style={
        inner sep=0pt,minimum size=14pt,
        circle, fill=white, draw, font=\small
  	},
  	big c node/.style={
        inner sep=0pt,minimum size=14pt,
        circle, fill=white, draw, font=\small
  	},
  	label node/.style={
        font=\small
  	}
]

	
    \node [v node] (v1) at (0,0) {$v_1$};
    \node [v node] (v2) [below of=v1, yshift=5] {$v_2$};
    \node [v node] (v3) [below of=v2, yshift=5] {$v_3$};
    \node [v node] (v4) [below of=v3, yshift=5] {$v_4$};
    \node [v node] (v5) [below of=v4, yshift=5] {$v_5$};
    \node [v node] (v6) [below of=v5, yshift=5] {$v_6$};

    \node [c node] (c1i) [right of=v1, xshift=20] {$c^1_i$};
    \node [c node] (c2i) [below of=c1i, yshift=-10] {$c^2_i$};
    \node [c node] (c3i) [below of=c2i, yshift=-10] {$c^3_i$};

    \node [big c node] (ci) [right of=c1i, xshift=20, yshift=-40] {$i$};
    \node [big c node] (ciprime) [below of=ci, xshift=0, yshift=-50] {$i'$};
    
    \node [st node] (sink) [right of=ci, xshift=20, yshift=-30] {$t$};

    \draw[->] (v1) -- (c1i);
    \draw[->] (v2) -- (c1i);
    \draw[->] (v2) -- (c2i);
    \draw[->] (v3) -- (c2i);
    \draw[->] (v4) -- (c3i);
    \draw[->] (v5) -- (c3i);
    \draw[->] (v6) -- (c3i);

    \draw[->] (c1i) -- (ci);
    \draw[->] (c2i) -- (ci);
    \draw[->] (c3i) -- (ci);

    \draw[->] (ci) -- (sink);
    \draw[->] (ciprime) -- (sink);


\end{tikzpicture}
\caption{Network flow construction.}\label{NF_diagram}
\end{figure}
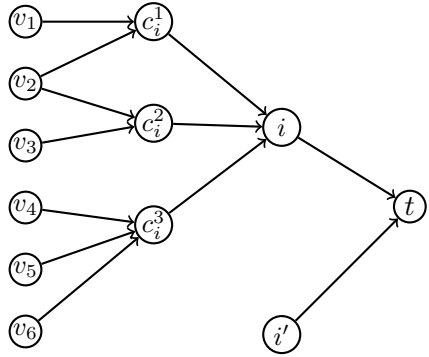

See Figure \ref{NF_diagram} for an example. It is clear that the entire demand is $|\Points|$ and that this is the maximum possible flow. The LP solution attains that flow. Further, since the demands, capacities and distances are integers, an optimal integral minimum-cost flow can be found in polynomial time. If $\Bar{x}_{ij}$ is the integer assignment that resulted from the flow computation, then violations are as follows:
\begin{theorem}\label{viol_2_th}
    The number of vertices assigned to a cluster (cluster size) is violated by at most 1, i.e. $\abs{\sum_{j \in \Points} \Bar{x}_{ij} - \sum_{j \in \Points} x_{ij}}\leq 1$. Further for metric membership, the violation in the average value is at most $2R$, i.e. $\abs{\sum_{j \in \Points} \Bar{x}_{ij} r_j- \sum_{j \in \Points} x_{ij}r_j} \leq 2R$. It follows that for the probabilistic case, the violation in the expected value is at most 2.  
\end{theorem}
\begin{proof}
For a given center $i$, every vertex $q \in C_i$ is assigned some vertices and adds value $\sum_{j \in \phi^{-1}(i,q)} R_j x^q_{ij}$ to the entire average (expected) value of cluster $i$ where $\phi^{-1}(i,q)$ refers to the subset in $\phi^{-1}(i)$ assigned to $q$. After the rounding,  $\sum_{j \in \phi^{-1}(i,q)} R_j x^q_{ij}$ will become $\sum_{j \in \phi^{-1}(i,q)} R_j \bar{x}^q_{ij}$. Denoting $\max_{j \in \phi^{-1}(i,q)} R_j$ and $\min_{j \in \phi^{-1}(i,q)} R_j$ by $R^{max}_{q,i}$ and $R^{min}_{q,i}$, respectively. The following bounds the maximum violation:
\begin{align*}
    & \sum_{q=1}^{|C_i|} \Big(\sum_{j \in \phi^{-1}(i,q)} R_j \bar{x}^q_{ij}\Big) - \sum_{q=1}^{|C_i|} \Big(\sum_{j \in \phi^{-1}(i,q)} R_j x^q_{ij}\Big)  \\ 
    & =\sum_{q=1}^{|C_i|} \sum_{j \in \phi^{-1}(i,q)} \Big( R_j \bar{x}^q_{ij}- R_j x^q_{ij} \Big) \\
    & \leq R^{max}_{|C_i|,i}  + \sum_{q=1}^{|C_i|-1} R^{max}_{q,i} - R^{min}_{q,i} \\
    & = R^{max}_{|C_i|,i}  + \Big( R^{max}_{1,i} - R^{min}_{1,i} \Big) + \Big( R^{max}_{2,i} - R^{min}_{2,i} \Big) \\
    & + \Big( R^{max}_{3,i} - R^{min}_{3,i} \Big) + \dots + \Big( R^{max}_{|C_i|-1,i} - R^{min}_{|C_i|-1,i} \Big)\\
    & \leq R^{max}_{|C_i|,i}  + \Big( R^{max}_{1,i} - R^{min}_{1,i} \Big) + \Big( R^{min}_{1,i} - R^{min}_{2,i} \Big) \\
    & + \Big( R^{min}_{2,i} - R^{min}_{3,i} \Big) + \dots + \Big( R^{min}_{|C_i|-2,i} - R^{min}_{|C_i|-1,i} \Big)\\
    & \leq R^{max}_{|C_i|,i}  + R^{max}_{1,i} - R^{min}_{|C_i|-1,i}\\
    & \leq 2R-0 =2R 
\end{align*}
where we invoked the fact that $R^{max}_{k,i} \leq R^{min}_{k-1,i}$. By a similar argument it can be shown that the maximum drop is $-2R$. For the probabilistic case, simply $R=1$. 
\end{proof}

Our rounding scheme results in a violation for the two color probabilistic case that is at most $2$, whereas for metric-membership it is $2R$. The violation of $2R$ for the metric membership case suggests that the rounding is too loose, therefore we show a lower bound of at least $\frac{R}{2}$ for any rounding scheme applied to the resulting solution. This also makes our rounding asymptotically optimal. 
\begin{theorem}\label{lb_round_th}
Any rounding scheme applied to the resulting solution has a fairness constraint violation of at least $\frac{R}{2}$ in the worst case. 
\end{theorem}
\begin{proof}
Consider the following instance (in Figure \ref{viol_lb}) with 5 points. Points 2 and 4 are chosen as the centers and both clusters have the same radius. The entire set has average color: $\frac{2(0)+2(\frac{3R}{4})+R}{2+2+1}=\frac{\frac{5R}{2}}{5}=\frac{R}{2}$. If the upper and lower values are set to $u=l=\frac{R}{2}$, then the fractional assignments for cluster 1 can be: $x_{21}=1,x_{22}=1,x_{23}=\frac{1}{2}$, leading to average color $\frac{\frac{3R}{4}+0+\frac{R}{2}}{1+1+\frac{1}{2}}=\frac{R}{2}$. For cluster 2 we would have: $x_{43}=\frac{1}{2},x_{44}=1,x_{45}=1$ and the average color is $\frac{R(\frac{3}{4}+\frac{1}{2})}{\frac{5}{2}}=\frac{\frac{5R}{4}}{\frac{5}{2}}=\frac{R}{2}$. Only assignments $x_{23}$ and $x_{43}$ are fractional and hence will be rounded. WLOG assume that $x_{23}=1$ and $x_{43}=0$. It follows that the change (violation) in the assignment $\sum_j r_j x_{ij}$ for a cluster $i$ will be $\frac{R}{2}$. Consider cluster 1, the resulting color is $\frac{3R}{4}+R=\frac{7R}{4}$, the change is $|\frac{7R}{4}-\frac{5R}{4}|=\frac{R}{2}$. Similarly, for cluster 2 the change is $|\frac{5R}{4}-\frac{3R}{4}|=\frac{R}{2}$. 

\begin{figure}[h!t]
\centering
\centerline{\includegraphics[width=0.45\textwidth]{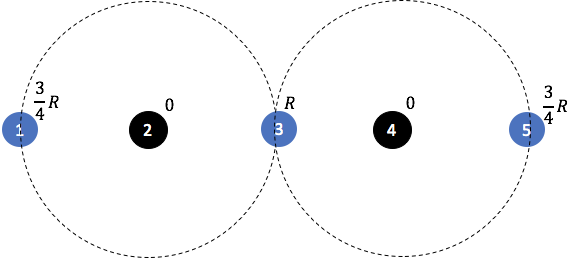}}
\caption{Points 2 and 4 have been selected as centers by the integer solution. Each points has its value written next to.}
\label{viol_lb}
\end{figure}

\end{proof}

\subsection{Algorithms for the Multiple Color Case Under a Large Cluster Assumption:}\label{mc_lc}
First, we point out that for the multi-color case, the algorithm is based on the assumption that the cluster size is large enough. Specifically:
\begin{assumption}\label{lc_assumption}
Each cluster in the optimal solution should have size at least $L=\Omega(n^r)$ where $r \in (0,1)$. 
\end{assumption}
We firmly believe that the above is justified in real datasets. Nonetheless, the ability to manipulate the parameter $r$, gives us enough flexibility to capture all occurring real-life scenarios. 
\begin{theorem}\label{whp_sample}
If Assumption \ref{lc_assumption} holds, then independent sampling results in the amount of color for each clusters to be concentrated around its expected value with high probability. 
\end{theorem}
\begin{proof}
First, each cluster $C_i$ has an amount of color $\pcolor$ equal to $S^{\pcolor}_{C_i}$ with $\E[S^{\pcolor}_{C_i}] = \sum_{\point \in C_i} p^{\pcolor}_{\point}$ according to theorem \ref{is_t1}. Furthermore, since the cluster is valid it follows that: $l_{\pcolor} \leq \E[S^{\pcolor}_{C_i}] \leq u_{\pcolor}$.  Define $l_{\min}=\min_{\pcolor \in \Colors}  \{l_{\pcolor}\} >0$, then for any $\delta \in [0,1]$ by Theorem \ref{chernoff} we have:
\begin{align*}
      &\Pr(|S^{\pcolor}_{C_i} -\E[S^{\pcolor}_{C_i}]| >  \delta \E[S^{\pcolor}_{C_i}] ) \leq 2e^{-\E[S^{\pcolor}_{C_i}] \delta^2/3}\\ 
      & \leq 2 \exp(- \frac{\delta^2}{3} \sum_{\point \in C_i} p^{\pcolor}_{\point} ) \leq 2 \exp(- \frac{\delta^2}{3} L l_{\min})
\end{align*}
This upper bounds the failure probability for a given cluster. For the entire set we use the union bound and get:
\begin{align*}
    &\Pr\Big(\Bigl\{\exists  i \in \{1,\dots,k\}, \pcolor \in \Colors \text{ s.t. } |S^{\pcolor}_{C_i} -\E[S^{\pcolor}_{C_i}]| >  \delta \E[S^{\pcolor}_{C_i}]\Bigl\}\Big)   \\
    & \leq 2k \lvert\Colors\rvert \exp(- \frac{\delta^2}{3} L l_{\min}) \leq 2 \frac{n}{L} \lvert\Colors\rvert\exp(- \frac{\delta^2}{3} L l_{\min})\\
    & \leq 2 \lvert\Colors\rvert n^{1-r} \exp(- \frac{\delta^2}{3}  l_{\min} n^{r})
\end{align*}
It is clear that given $r, \delta$, and $l_{\min}$ there exists a constant $c$ such that the above is bounded by $\frac{1}{n^c}$. Therefore, the result holds with high probability. 
\end{proof}

Given Theorem~\ref{whp_sample} our solution essentially forms a reduction from the problem of probabilistic fair clustering $\PFC(k,p)$ to the problem of deterministic fair clustering with lower bounded cluster sizes which we denote by $\DFCLB(k,p,L)$ (the color assignments are known deterministically and each cluster is constrained to have size at least $L$). 
\begin{algorithm}\label{}
   \caption{Algorithm for Large Cluster $\PFC(k,p)$}
   \label{alg:large_pfc_alg}
\begin{algorithmic}
   \STATE {\bfseries Input:} $\Points, d, k, p, L, \{(l_{\pcolor},u_{\pcolor})\}_{\pcolor \in \Colors}$
   \STATE Relax the upper and lower by $\epsilon$: $\forall \pcolor \in \Colors,$ $l_{\pcolor} \leftarrow l_{\pcolor}(1-\epsilon)$ and $u_{\pcolor} \leftarrow u_{\pcolor}(1+\epsilon)$
   \STATE For each point $\point \in \Points$ sample its color independently according to $p^{\pcolor}_{\point}$ 
   \STATE Solve the deterministic fair clustering problem with lower bounded clusters $\DFCLB(k,p,L)$ over the generated instance and return the solution. 
\end{algorithmic}
\end{algorithm}
Our algorithm (\ref{alg:large_pfc_alg}) involves three steps. In the first step, the upper and lower bounds are relaxed since -although we have high concentration guarantees around the expectation- in the worst case the expected value may not be realizable (could not be an integer). Moreover the upper and lower bounds could coincide with the expected value causing violations of the bounds with high probability. See appendix \ref{is_details} for more details.

After that, the color assignments are sampled independently. The following deterministic fair clustering problem is solved for resulting set of points:
\begin{subequations}\label{dfclb}
  \begin{equation}
     \min_{S: |S|\leq k,\phi} L^k_p(\Points) \\  \end{equation}
  \begin{equation}
    \label{dfclb_eq_2}
     \text{s.t. } \forall i \in S: (1-\delta)l_{\pcolor} |\Points_i| \leq |\Points_{i,{\pcolor}}| \leq (1+\delta)u_{\pcolor} |\Points_i|
  \end{equation}
  \begin{equation}
    \label{dfclb_eq_3}
         \forall i \in S:  |\Points_i| \ge L
  \end{equation}
\end{subequations}
The difference between the original deterministic fair clustering problem and the above is that the bounds are relaxed by $\epsilon$ and a lower bound $L$ is required on the cluster size. This is done in order to guarantee that the resulting solution satisfies the relaxed upper and lower bounds in expectation, because small size clusters do not have a Chernoff bound and therefore nothing ensures that they are valid solutions to the original $\PFC(k,p)$ problem. 

The algorithm for solving deterministic fair clustering with lower bounded cluster sizes $\DFCLB$ is identical to the algorithm for solving the original deterministic fair clustering \cite{bera2019fair,bercea2018cost} problem with the difference being that the setup LP will have a bound on the cluster size. That is we include the following constraint $\forall i \in S: \sum_{ij} x_{ij} \ge L$. However, the lower bound on the cluster size causes an issue, since it is possible that a center or centers from the color-blind solution need to be closed. Therefore, we have to try all possible combinations for opening and closing the centers. Since there are at most $2^k$ possibilities, this leads to a run-time that is fixed parameter tractable $O(2^k \text{poly}(n))$. In theorem \ref{lb_appendix} we show that this leads to an approximation ratio of $\alpha+2$ like the ordinary (deterministic) fair clustering case, where again $\alpha$ is the approximation ratio of the color blind algorithm. See also appendix \ref{dfc_lb_appendix} for further details.

\begin{theorem}\label{last-thm}
Given an instance of the probabilistic fair clustering problem $\PFC(k,p)$, with high probability algorithm \ref{alg:large_pfc_alg} results in a solution with violation at most $\epsilon$ and approximation ratio $(\alpha+2)$ in $O(2^k \text{poly}(n))$ time. 
\end{theorem}
\begin{proof}
First, given an instance $\mathcal{I}_{\PFC}$ of probabilistic fair clustering with optimal value $\mathrm{OPT}_{\PFC}$ the clusters in the optimal solution would with high probability be a valid solution for the deterministic setting, as showed in Theorem \ref{whp_sample}. Moreover the objective value of the solution is unchanged. Therefore, the resulting deterministic instance would have $\OPT_{\DFCLB} \leq  \OPT_{\PFC}$. Hence, the algorithm will return a solution with cost at most $(\alpha+2) \OPT_{\DFCLB} \leq  (\alpha+2) \OPT_{\PFC}$. 

For the solution $\mathrm{SOL}_{\DFCLB}$ returned by the algorithm, each cluster is of size at least $L$, and the Chernoff bound guarantees that the violation in expectation is at most $\epsilon$ with high probability. 

The run-time comes from the fact  that $\DFCLB$ is solved in $O(2^k \text{poly}(n))$ time. 
\end{proof}

\section{Experiments}\label{sec:experiments}
We now evaluate the performance of our algorithms over a collection of real-world datasets.  We give experiments in the two (unordered) color case (\S\ref{sec:experiments-two-color}), metric membership (i.e., ordered color) case (\S\ref{sec:experiments-metric}), as well as under the large cluster assumption (\S\ref{sec:experiments-large-cluster}).  We include experiments for the $k$-means case here, and the (qualitatively similar) $k$-center and $k$-median experiments to Appendix~\ref{further_exps}.

\subsection{Experimental Framework}\label{sec:experiments-framework}
\noindent\textbf{Hardware \& Software.}  We used only commodity hardware through the experiments: Python 3.6 on a MacBook Pro with 2.3GHz Intel Core i5 processor and 8GB 2133MHz LPDDR3 memory. A state-of-the-art commercial optimization toolkit, \texttt{CPLEX}~\cite{manual2016version}, was used for solving all linear programs (LPs). 
\texttt{NetworkX}~\cite{hagberg2013networkx} was used to solve minimum cost flow problems, and \texttt{Scikit-learn}~\cite{pedregosa2011scikit} was used for standard machine learning tasks such as training SVMs, pre-processing, and performing traditional $k$-means clustering. 

\noindent\textbf{Color-Blind Clustering.} The color-blind clustering algorithms we use are as follows.
\begin{itemize}[nosep]
\item \cite{gonzalez1985clustering} gives a $2$-approximation for $k$-center.
\item We use \texttt{Scikit-learn}'s $k$-means++ module.
\item We use the $5$-approximation algorithm due to \cite{arya2004local} modified with $D$-sampling \cite{arthur2006k} according to~\cite{bera2019fair}. 
\end{itemize}

\textbf{Generic-Experimental Setup and Measurements.} For a chosen dataset, a given color $\pcolor$ would have a proportion $f_{\pcolor}=\frac{|{v \in \Points| \chi(v)=\pcolor}|}{|\Points|}$. Following~\cite{bera2019fair}, the lower bound is set to $l_{\pcolor}=(1-\delta)r_{\pcolor}$ and the upper bound is to $u_{\pcolor}=\frac{f_{\pcolor}}{(1-\delta)}$. For metric membership, we similarly have $f=\frac{\sum_{\point \in \Points} r_{\point}}{|\Points|}$ as the proportion, $l=(1-\delta)f$ and $u=\frac{f}{1-\delta}$ as the lower and upper bound, respectively. We set $\delta=0.2$, as~\cite{bera2019fair} did, unless stated otherwise.

For each experiment, we measure the price of fairness $\POF=\frac{\textrm{Fair Solution Cost}}{\textrm{Color-Blind Cost}}$. We also measure the maximum additive violation $\gamma$ as it appears in inequalities \ref{viol_pfc} and \ref{viol_mm}.

\subsection{Two Color Case}\label{sec:experiments-two-color}
Here we test our algorithm for the case of two colors with probabilistic assignment. We use the \bank{} dataset \cite{moro2014data} which has 4,521 data points. We choose marital status, a categorical variable, as our fairness (color) attribute. To fit the binary color case, we merge single and divorced into one category.  Similar to the supervised learning work due to~\cite{awasthi2019effectiveness}, we make \bank{}'s deterministic color assignments probabilistic by independently perturbing them for each point with probability $\pn$. Specifically, if $\point$ originally had color $c_{\point}$, then now it has color $c_{\point}$ with probability $1-\pn$ instead. To make the results more interpretable, we define $\pacc=1-\pn$. Clearly, $\pacc=1$ corresponds to the deterministic case, and $\pacc=\frac{1}{2}$ corresponds to completely random assignments.  

First, in Fig.~\ref{2_color}(a), we see that the violations of the color-blind solution can be as large as 25 whereas our algorithm is within the theoretical guarantee that is less than 1. In Fig.~\ref{2_color}(b), we see that in spite of the large violation, fairness can be achieved at a low relative efficiency loss, not exceeding 2\% ($\POF{} \leq 1.02$).

\begin{figure}[h!t]
\centering
\centerline{\includegraphics[width=0.9\columnwidth]{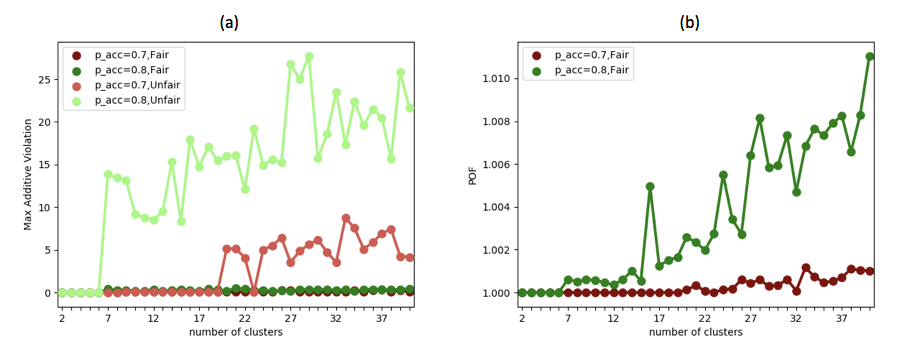}}
\caption{For $\pacc=0.7$ \& $\pacc=0.8$, showing (a): \#clusters vs.\ maximum additive violation; (b): \#clusters vs.\ $\POF{}$.}\label{2_color}
\end{figure}

How does labeling accuracy level $\pacc$ impact this problem? Fig.~\ref{p_acc_vs_POF} shows $\pacc$ vs $\POF$ for $\delta=0.2$ and $\delta=0.1$. At $\pacc=\frac{1}{2}$, color assignments are completely random and the cost is, as expected, identical to color-blind cost. As $\pacc$ increases, the colors of the vertices become more differentiated, causing $\POF$ to increase, eventually reaching the maximum at $\pacc=1$ which is the deterministic case.

\begin{figure}[h!t]
\centering
\centerline{\includegraphics[width=\columnwidth]{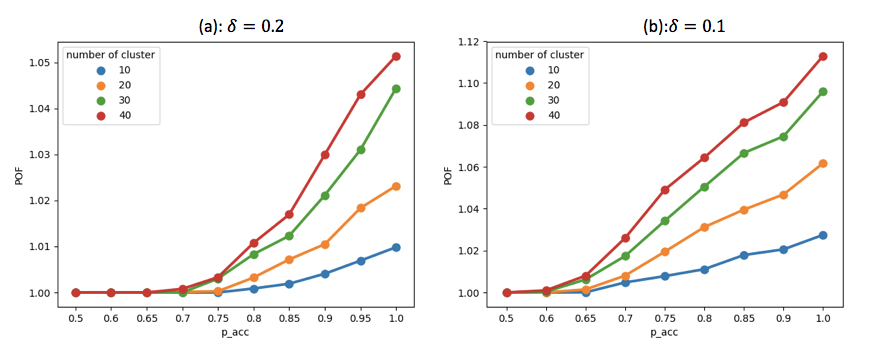}}
\caption{Plot showing $\pacc$ vs $\POF$, (a):$\delta=0.2$ and (b):$\delta=0.1$.}
\label{p_acc_vs_POF}
\end{figure}

Next, we test against an ``obvious'' strategy when faced with probabilistic color labels: simply \emph{threshold} the probability values, and then run a deterministic fair clustering algorithm. Fig.~\ref{thresh_pfc}(a) shows that this may indeed work for guaranteeing fairness, as the proportions may be satisfied with small violations; however, it comes at the expense of a much higher $\POF$.  Fig.~\ref{thresh_pfc}(b) supports this latter statement: our algorithm can achieve the same violations with smaller $\POF$. Further, running a deterministic algorithm over the thresholded instance may result in an infeasible problem.\footnote{An intuitive example of infeasibility: consider the two color case where $p_{\point}=\frac{1}{2}+\epsilon, \forall \point \in \Points$ for some small positive $\epsilon$. Thresholding drastically changes the overall probability to $1$; therefore no subset of points would have proportion around $\frac{1}{2}+\epsilon$.} 
 
\begin{figure}[!ht]
\centering
\centerline{\includegraphics[width=\columnwidth]{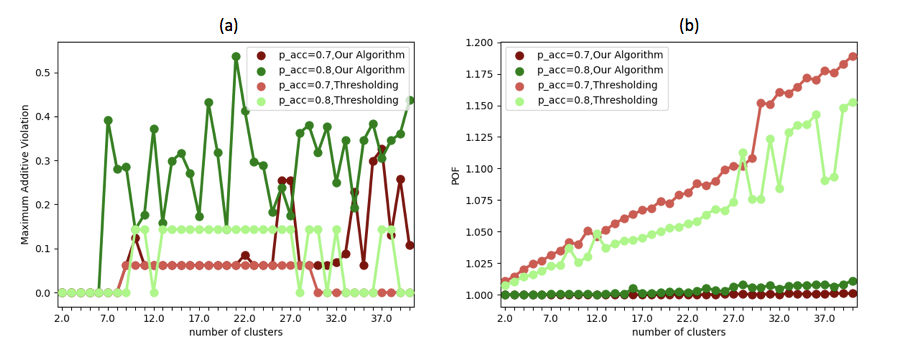}}
\caption{Comparing our algorithm to thresholding followed by deterministic fair clustering: (a)maximum violation, (b) $\POF$.}
\label{thresh_pfc}
\end{figure}

\subsection{Metric Membership}\label{sec:experiments-metric}
Here we test our algorithm for the metric membership problem. We use two additional well-known datasets: $\adult{}$~\cite{kohavi1996scaling}, with age being the fairness attribute, and $\credit{}$~\cite{yeh2009comparisons}, with credit being the fairness attribute. We apply a pre-processing step where for each point we subtract the minimum value of the fairness attribute over the entire set. This has the affect of reducing the maximum fairness attribute value, therefore reducing the maximum possible violation of $\frac{1}{2}R$, but still keeping the values non-negative. 

Fig.~\ref{POF__exp_mm} shows $\POF$ with respect to the number of clusters. For the \adult{} dataset, $\POF$ is at most less than 5\%, whereas for the \credit{} dataset it is as high at 25\%. While the $\POF$, intuitively, rises with the number of clusters allowed, it is substantially higher with the \credit{} dataset.  This may be explained because of the correlation that exists between credit and other features represented in the metric space.

\begin{figure}[!ht]
\vskip 0.2in
\begin{center}
\centerline{\includegraphics[width=0.7\columnwidth]{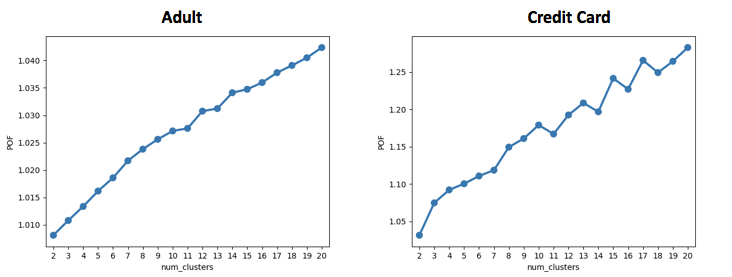}}
\caption{Plot showing the number of clusters vs $\POF$}
\label{POF__exp_mm}
\end{center}
\vskip -0.2in
\end{figure}

In Fig.~\ref{viol_exp_mm}, we compare the number of clusters against the normalized maximum additive violation. The normalized maximum additive violation is the same maximum additive violation $\gamma$ from inequality \ref{viol_mm}---but normalized by $R$. We see that the normalized maximum additive violation is indeed less than $2$ as theoretically guaranteed by our algorithm, whereas for the color-blind solution it is as high a $250$. 
\begin{figure}[!ht]
\vskip 0.2in
\begin{center}
\centerline{\includegraphics[width=0.7\columnwidth]{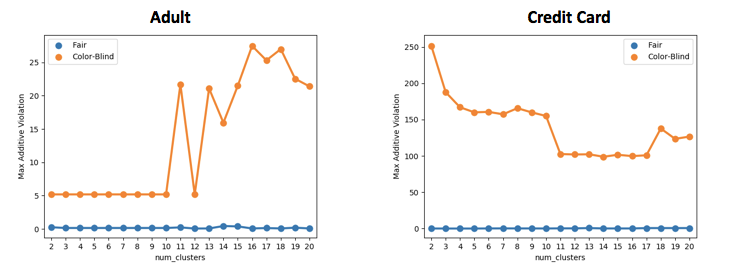}}
\caption{Plot showing the number of clusters vs the normalized maximum additive violation}
\label{viol_exp_mm}
\end{center}
\vskip -0.2in
\end{figure}

\subsection{The Large Cluster Assumption}\label{sec:experiments-large-cluster}

Here we test our algorithm for the case of probabilistically assigned multiple colors under Assumption~\ref{lc_assumption}, which addresses cases where the optimal clustering does not include pathologically small clusters. We use the $\cens$~\cite{meek2002learning} dataset.  We note that $\cens$ is large, with over 2.4 million points.  We use age groups (attribute \texttt{dAge} in the dataset) as our fairness attribute, which yields $7$ age groups (colors).\footnote{Group $0$ is extremely rare, to the point that it violates the ``large cluster'' assumption for most experiments; therefore, we merged it with Group $1$, its nearest age group.}  We then sample 100,000 data points and use them to train an SVM classifier\footnote{We followed standard procedures and ended up with a standard RBF-based SVM; the accuracy of this SVM is somewhat orthogonal to the message of this paper, and rather serves to illustrate a real-world, noisy labeler.} to predict the age group memberships. The classifier achieves an accuracy of around 68\%. We use the classifier to predict the memberships of another 100,000 points not included in the training set, and sample from that to form the probabilistic assignment of colors. Although as stated earlier we should try all possible combinations in closing and opening the color-blind centers, we keep all centers as they are. It is expected that this heuristic would not lead to a much higher cost if the dataset and the choice of the color-blind centers is sufficiently well-behaved. 

Fig.~\ref{exp_lc} shows the output of our large cluster algorithm over 100,000 points and $k=5$ clusters with varying lower bound assumptions. Since the clusters here are large, we normalize the additive violations by the cluster size. We see that our algorithm results in normalized violation that decrease as the lower bound on the cluster size increases. The $\POF$ is high relative to our previous experiments, but still less than 50\%.

\begin{figure}[H]
\begin{center}
\centerline{\includegraphics[width=.9\columnwidth]{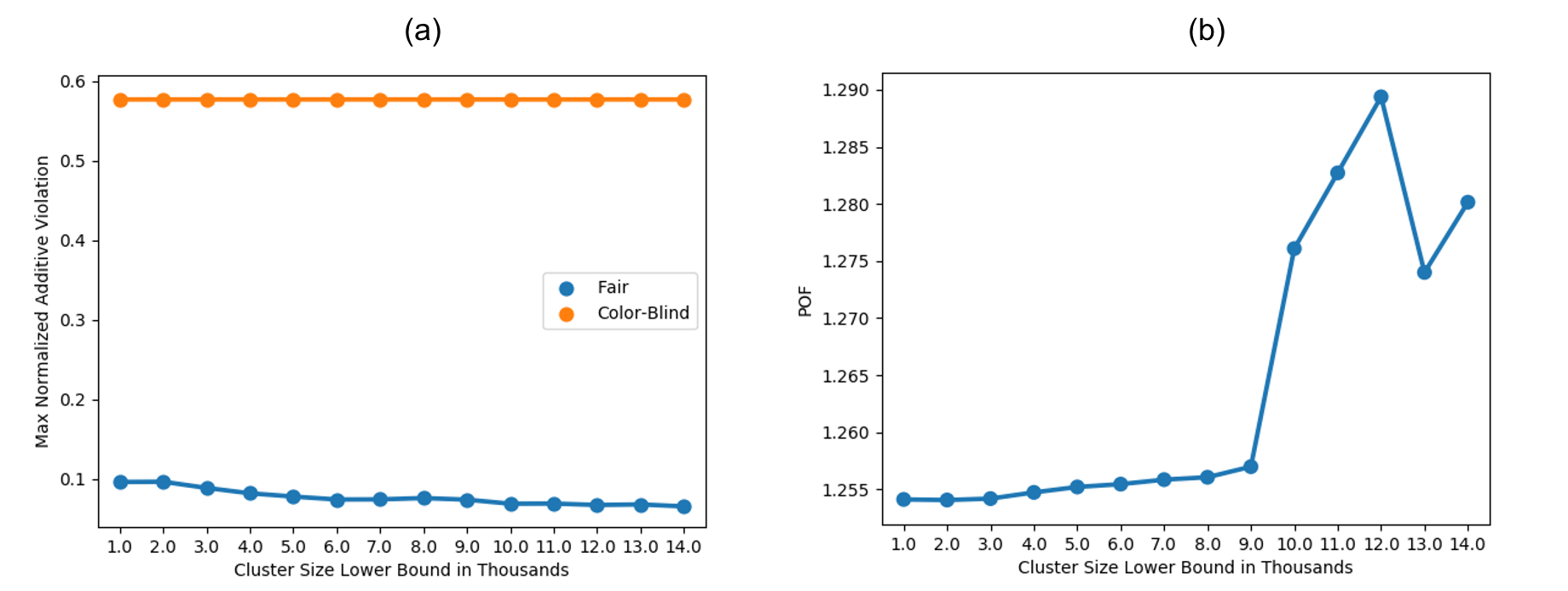}}
\caption{Plot showing the performance of our independent sampling algorithm over the $\cens$ dataset for $k=5$ clusters with varying values on the cluster size lower bound:(a)maximum violation normalized by the cluster size, (b)the price of fairness.}
\label{exp_lc}
\end{center}
\end{figure}

\section{Conclusions \& Future Research}\label{sec:conclusions}
Prior research in fair clustering assumes deterministic knowledge of group membership. We generalized prior work by assuming probabilistic knowledge of group membership. In this new model, we presented novel clustering algorithms in this more general setting with approximation ratio guarantees.  We also addressed the problem of ``metric membership,'' where different groups have a notion of order and distance---this addresses real-world use cases where parity must be ensured over, e.g., age or income. We also conducted experiments on slate of datasets.  The algorithms we propose come with strong theoretical guarantees; on real-world data, we showed that those guarantees are easily met.
Future research directions involve the assignment of \emph{multiple colors} (e.g., race as well as self-reported gender) to vertices, in addition to the removal of assumptions such as the large cluster assumption.


\nocite{langley00}

\section{Broader Impact}

Guaranteeing that the color proportions are maintained in each cluster satisfies group (demographic) fairness in clustering. In real-world scenarios, however, group membership may not be known with certainty but rather probabilistically (e.g., learned by way of a machine learning model). Our paper addresses fair clustering in such a scenario and therefore both generalizes that particular (and well-known) problem statement and widens the scope of the application.  In settings where a group-fairness-aware clustering algorithm is appropriate to deploy, we believe our work could increase the robustness of those systems.  That said, we do note (at least) two broader points of discussion that arise when placing potential applications of our work in the greater context of society:
\begin{itemize}
    \item We address a specific definition of fairness.  While the formalization we address is a common one that draws directly on legal doctrine such as the notion of disparate impact, as expressed by~\cite{Feldman15:Certifying} and others, we note that the Fairness, Accountability, Transparancy, and Ethics (FATE) in machine learning community has identified \emph{many} such definitions~\cite{Verma18:Fairness}.  Yet, there is a growing body of work exploring the gaps between the FATE-style definitions of fairness and those desired in industry (see, e.g., recent work due to~\citet{Holstein19:Improving} that interviews developers about their wants and needs in this space), and there is growing evidence that stakeholders may not even comprehend those definitions in the first place~\citet{Saha20:Measuring}.  Indeed, ``deciding on a definition of fairness'' is an inherently morally-laden, application-specific decision, and we acknowledge that making a prescriptive statement about whether or not our model is \emph{appropriate} for a particular use case is the purview of both technicians, such as ourselves, and policymakers and/or other stakeholders.
    \item Our work is motivated by the assumption that, in many real-world settings, group membership may not be known deterministically.  If group membership is being estimated by a machine-learning-based model, then it is likely that this estimator itself could incorporate bias into the membership estimate; thus, our final clustering could also reflect that bias.  As an example, take a bank in the United States; here, it may not be legal for a bank to store information on sensitive attributes---a fact made known recently by the ``Apple Card'' fiasco of late 2019~\cite{Knight19:Apple}.  Thus, to audit algorithms for bias, it may be the case that either the bank or a third-party service infers sensitive attributes from past data, which likely introduces bias into the group membership estimate itself.  (See recent work due to~\citet{Chen19:Fairness} for an in-depth discussion from the point of view of an industry-academic team.)
\end{itemize}
We have tried to present this work without making normative statements about, e.g., the definition of fairness used; still, we emphasize the importance of open dialog with stakeholders in any system, and acknowledge that our proposed approach serves as one part of a larger application ecosystem.

\section*{Acknowledgments}
Dickerson and Esmaeili were supported in part by NSF CAREER Award IIS-1846237, DARPA GARD Award \#HR112020007, DARPA SI3-CMD Award \#S4761, DoD WHS Award \#HQ003420F0035, NIH R01 Award NLM-013039-01, and a Google Faculty Research Award.  We thank Keegan Hines for discussion about ``fairness under unawareness'' in practical settings, and for pointers to related literature.

%
%
%
\newpage
 \bibliographystyle{plainnat}
 \bibliography{arxiv_draft}

\begin{thebibliography}{45}
\providecommand{\natexlab}[1]{#1}
\providecommand{\url}[1]{\texttt{#1}}
\expandafter\ifx\csname urlstyle\endcsname\relax
  \providecommand{\doi}[1]{doi: #1}\else
  \providecommand{\doi}{doi: \begingroup \urlstyle{rm}\Url}\fi

\bibitem[Aggarwal and Reddy(2013)]{aggarwal2013data}
Charu~C Aggarwal and Chandan~K Reddy.
\newblock Data clustering: Algorithms and applications.
\newblock 2013.

\bibitem[Aggarwal et~al.(2010)Aggarwal, Panigrahy, Feder, Thomas, Kenthapadi,
  Khuller, and Zhu]{aggarwal2010achieving}
Gagan Aggarwal, Rina Panigrahy, Tom{\'a}s Feder, Dilys Thomas, Krishnaram
  Kenthapadi, Samir Khuller, and An~Zhu.
\newblock Achieving anonymity via clustering.
\newblock \emph{ACM Transactions on Algorithms (TALG)}, 6\penalty0
  (3):\penalty0 49, 2010.

\bibitem[Ahmadian et~al.(2019{\natexlab{a}})Ahmadian, Epasto, Kumar, and
  Mahdian]{ahmadian2019clustering}
Sara Ahmadian, Alessandro Epasto, Ravi Kumar, and Mohammad Mahdian.
\newblock Clustering without over-representation.
\newblock \emph{arXiv preprint arXiv:1905.12753}, 2019{\natexlab{a}}.

\bibitem[Ahmadian et~al.(2019{\natexlab{b}})Ahmadian, Norouzi-Fard, Svensson,
  and Ward]{ahmadian2019better}
Sara Ahmadian, Ashkan Norouzi-Fard, Ola Svensson, and Justin Ward.
\newblock Better guarantees for k-means and euclidean k-median by primal-dual
  algorithms.
\newblock \emph{SIAM Journal on Computing}, \penalty0 (0):\penalty0 FOCS17--97,
  2019{\natexlab{b}}.

\bibitem[An et~al.(2015)An, Bhaskara, Chekuri, Gupta, Madan, and
  Svensson]{an2015centrality}
Hyung-Chan An, Aditya Bhaskara, Chandra Chekuri, Shalmoli Gupta, Vivek Madan,
  and Ola Svensson.
\newblock Centrality of trees for capacitated k-center.
\newblock \emph{Mathematical Programming}, 154\penalty0 (1-2):\penalty0 29--53,
  2015.

\bibitem[Angwin et~al.(2016)Angwin, Larson, Mattu, and
  Kirchner]{angwin2016machine}
Julia Angwin, Jeff Larson, Surya Mattu, and Lauren Kirchner.
\newblock Machine bias.
\newblock \emph{ProPublica, May}, 23:\penalty0 2016, 2016.

\bibitem[Arthur and Vassilvitskii(2006)]{arthur2006k}
David Arthur and Sergei Vassilvitskii.
\newblock k-means++: The advantages of careful seeding.
\newblock Technical report, Stanford, 2006.

\bibitem[Arya et~al.(2004)Arya, Garg, Khandekar, Meyerson, Munagala, and
  Pandit]{arya2004local}
Vijay Arya, Naveen Garg, Rohit Khandekar, Adam Meyerson, Kamesh Munagala, and
  Vinayaka Pandit.
\newblock Local search heuristics for k-median and facility location problems.
\newblock \emph{SIAM Journal on computing}, 33\penalty0 (3):\penalty0 544--562,
  2004.

\bibitem[Awasthi et~al.(2019)Awasthi, Kleindessner, and
  Morgenstern]{awasthi2019effectiveness}
Pranjal Awasthi, Matth{\"a}us Kleindessner, and Jamie Morgenstern.
\newblock Effectiveness of equalized odds for fair classification under
  imperfect group information.
\newblock \emph{arXiv preprint arXiv:1906.03284}, 2019.

\bibitem[Backurs et~al.(2019)Backurs, Indyk, Onak, Schieber, Vakilian, and
  Wagner]{backurs2019scalable}
Arturs Backurs, Piotr Indyk, Krzysztof Onak, Baruch Schieber, Ali Vakilian, and
  Tal Wagner.
\newblock Scalable fair clustering.
\newblock In \emph{International Conference on Machine Learning}, pages
  405--413, 2019.

\bibitem[Bera et~al.(2019)Bera, Chakrabarty, and Negahbani]{bera2019fair}
Suman~K Bera, Deeparnab Chakrabarty, and Maryam Negahbani.
\newblock Fair algorithms for clustering.
\newblock \emph{arXiv preprint arXiv:1901.02393}, 2019.

\bibitem[Bercea et~al.(2018)Bercea, Gro{\ss}, Khuller, Kumar, R{\"o}sner,
  Schmidt, and Schmidt]{bercea2018cost}
Ioana~O Bercea, Martin Gro{\ss}, Samir Khuller, Aounon Kumar, Clemens
  R{\"o}sner, Daniel~R Schmidt, and Melanie Schmidt.
\newblock On the cost of essentially fair clusterings.
\newblock \emph{arXiv preprint arXiv:1811.10319}, 2018.

\bibitem[Biddle(2006)]{biddle2006adverse}
Dan Biddle.
\newblock \emph{Adverse impact and test validation: A practitioner's guide to
  valid and defensible employment testing}.
\newblock Gower Publishing, Ltd., 2006.

\bibitem[Bogen and Rieke(2018)]{BoRi18a}
M.~Bogen and A.~Rieke.
\newblock Help wanted: {A}n examination of hiring algorithms, equity, and bias.
\newblock Technical report, Upturn, 2018.

\bibitem[Byrka et~al.(2014)Byrka, Pensyl, Rybicki, Srinivasan, and
  Trinh]{byrka2014improved}
Jaros{\l}aw Byrka, Thomas Pensyl, Bartosz Rybicki, Aravind Srinivasan, and Khoa
  Trinh.
\newblock An improved approximation for k-median, and positive correlation in
  budgeted optimization.
\newblock In \emph{Proceedings of the twenty-sixth annual ACM-SIAM symposium on
  Discrete algorithms}, pages 737--756. SIAM, 2014.

\bibitem[Chen et~al.(2019)Chen, Kallus, Mao, Svacha, and
  Udell]{Chen19:Fairness}
Jiahao Chen, Nathan Kallus, Xiaojie Mao, Geoffry Svacha, and Madeleine Udell.
\newblock Fairness under unawareness: Assessing disparity when protected class
  is unobserved.
\newblock In \emph{Proceedings of the Conference on Fairness, Accountability,
  and Transparency (FAT*)}, pages 339--348, 2019.

\bibitem[Chierichetti et~al.(2017)Chierichetti, Kumar, Lattanzi, and
  Vassilvitskii]{chierichetti2017fair}
Flavio Chierichetti, Ravi Kumar, Silvio Lattanzi, and Sergei Vassilvitskii.
\newblock Fair clustering through fairlets.
\newblock In \emph{Advances in Neural Information Processing Systems}, pages
  5029--5037, 2017.

\bibitem[Cygan et~al.(2012)Cygan, Hajiaghayi, and Khuller]{cygan2012lp}
Marek Cygan, MohammadTaghi Hajiaghayi, and Samir Khuller.
\newblock Lp rounding for k-centers with non-uniform hard capacities.
\newblock In \emph{2012 IEEE 53rd Annual Symposium on Foundations of Computer
  Science}, pages 273--282. IEEE, 2012.

\bibitem[Feldman et~al.(2015)Feldman, Friedler, Moeller, Scheidegger, and
  Venkatasubramanian]{Feldman15:Certifying}
Michael Feldman, Sorelle~A Friedler, John Moeller, Carlos Scheidegger, and
  Suresh Venkatasubramanian.
\newblock Certifying and removing disparate impact.
\newblock In \emph{International Conference on Knowledge Discovery and Data
  Mining (KDD)}, pages 259--268, 2015.

\bibitem[Gandhi et~al.(2006)Gandhi, Khuller, Parthasarathy, and
  Srinivasan]{gandhi2006dependent}
Rajiv Gandhi, Samir Khuller, Srinivasan Parthasarathy, and Aravind Srinivasan.
\newblock Dependent rounding and its applications to approximation algorithms.
\newblock \emph{Journal of the ACM (JACM)}, 53\penalty0 (3):\penalty0 324--360,
  2006.

\bibitem[Gonzalez(1985{\natexlab{a}})]{Gonzalez1985}
Teofilo~F. Gonzalez.
\newblock Clustering to minimize the maximum intercluster distance.
\newblock \emph{Theoretical Computer Science}, 1985{\natexlab{a}}.
\newblock ISSN 0304-3975.

\bibitem[Gonzalez(1985{\natexlab{b}})]{gonzalez1985clustering}
Teofilo~F Gonzalez.
\newblock Clustering to minimize the maximum intercluster distance.
\newblock \emph{Theoretical Computer Science}, 38:\penalty0 293--306,
  1985{\natexlab{b}}.

\bibitem[Hagberg et~al.(2013)Hagberg, Schult, Swart, Conway,
  S{\'e}guin-Charbonneau, Ellison, Edwards, and Torrents]{hagberg2013networkx}
Aric Hagberg, Dan Schult, Pieter Swart, D~Conway, L~S{\'e}guin-Charbonneau,
  C~Ellison, B~Edwards, and J~Torrents.
\newblock Networkx. high productivity software for complex networks.
\newblock \emph{Webov{\'a} str{\'a} nka https://networkx. lanl. gov/wiki},
  2013.

\bibitem[Hardt et~al.(2016)Hardt, Price, and Srebro]{Hardt16:Equality}
Moritz Hardt, Eric Price, and Nathan Srebro.
\newblock Equality of opportunity in supervised learning.
\newblock In \emph{Proceedings of the 30th International Conference on Neural
  Information Processing Systems}, NIPS’16, page 3323–3331, Red Hook, NY,
  USA, 2016. Curran Associates Inc.
\newblock ISBN 9781510838819.

\bibitem[Hochbaum and Shmoys(1985)]{Hochbaum1985}
Dorit~S. Hochbaum and David~B. Shmoys.
\newblock A best possible heuristic for the k-center problem.
\newblock \emph{Math. Oper. Res.}, May 1985.
\newblock ISSN 0364-765X.

\bibitem[Hochbaum and Shmoys(1986)]{Hochbaum1986}
Dorit~S. Hochbaum and David~B. Shmoys.
\newblock A unified approach to approximation algorithms for bottleneck
  problems.
\newblock \emph{J. ACM}, May 1986.
\newblock ISSN 0004-5411.

\bibitem[Holstein et~al.(2019)Holstein, Wortman~Vaughan, Daum{\'e}~III, Dudik,
  and Wallach]{Holstein19:Improving}
Kenneth Holstein, Jennifer Wortman~Vaughan, Hal Daum{\'e}~III, Miro Dudik, and
  Hanna Wallach.
\newblock Improving fairness in machine learning systems: What do industry
  practitioners need?
\newblock In \emph{Proceedings of the Conference on Human Factors in Computing
  Systems (CHI)}, 2019.

\bibitem[Huang et~al.(2019)Huang, Jiang, and Vishnoi]{NIPS2019_8976}
Lingxiao Huang, Shaofeng Jiang, and Nisheeth Vishnoi.
\newblock Coresets for clustering with fairness constraints.
\newblock In \emph{Advances in Neural Information Processing Systems 32}, pages
  7587--7598. Curran Associates, Inc., 2019.

\bibitem[Joseph et~al.(2016)Joseph, Kearns, Morgenstern, and
  Roth]{Joseph16:Fairness}
Matthew Joseph, Michael Kearns, Jamie Morgenstern, and Aaron Roth.
\newblock Fairness in learning: Classic and contextual bandits.
\newblock In \emph{Proceedings of the 30th International Conference on Neural
  Information Processing Systems}, NIPS’16, page 325–333, Red Hook, NY,
  USA, 2016. Curran Associates Inc.
\newblock ISBN 9781510838819.

\bibitem[Kallus et~al.(2019)Kallus, Mao, and Zhou]{kallus2019assessing}
Nathan Kallus, Xiaojie Mao, and Angela Zhou.
\newblock Assessing algorithmic fairness with unobserved protected class using
  data combination.
\newblock \emph{arXiv preprint arXiv:1906.00285}, 2019.

\bibitem[Khuller and Sussmann(2000)]{khuller2000capacitated}
Samir Khuller and Yoram~J Sussmann.
\newblock The capacitated k-center problem.
\newblock \emph{SIAM Journal on Discrete Mathematics}, 13\penalty0
  (3):\penalty0 403--418, 2000.

\bibitem[Knight(2019)]{Knight19:Apple}
Will Knight.
\newblock The {A}pple {C}ard didn't `see' gender---and that's the problem.
\newblock \emph{Wired}, 2019.

\bibitem[Kohavi(1996)]{kohavi1996scaling}
Ron Kohavi.
\newblock Scaling up the accuracy of naive-bayes classifiers: A decision-tree
  hybrid.
\newblock In \emph{Kdd}, volume~96, pages 202--207, 1996.

\bibitem[Langley(2000)]{langley00}
P.~Langley.
\newblock Crafting papers on machine learning.
\newblock In Pat Langley, editor, \emph{Proceedings of the 17th International
  Conference on Machine Learning (ICML 2000)}, pages 1207--1216, Stanford, CA,
  2000. Morgan Kaufmann.

\bibitem[Ledford(2019)]{ledford2019millions}
Heidi Ledford.
\newblock Millions of black people affected by racial bias in health-care
  algorithms.
\newblock \emph{Nature}, 574\penalty0 (7780):\penalty0 608, 2019.

\bibitem[Luong et~al.(2011)Luong, Ruggieri, and Turini]{Luong11:KNN}
Binh~Thanh Luong, Salvatore Ruggieri, and Franco Turini.
\newblock K-nn as an implementation of situation testing for discrimination
  discovery and prevention.
\newblock In \emph{Proceedings of the 17th ACM SIGKDD International Conference
  on Knowledge Discovery and Data Mining}, KDD ’11, page 502–510, New York,
  NY, USA, 2011. Association for Computing Machinery.
\newblock ISBN 9781450308137.
\newblock \doi{10.1145/2020408.2020488}.
\newblock URL \url{https://doi.org/10.1145/2020408.2020488}.

\bibitem[Manual(2016)]{manual2016version}
IBM CPLEX~User’s Manual.
\newblock Version 12 release 7.
\newblock \emph{IBM ILOG CPLEX Optimization}, 2016.

\bibitem[Meek et~al.(2002)Meek, Thiesson, and Heckerman]{meek2002learning}
Christopher Meek, Bo~Thiesson, and David Heckerman.
\newblock The learning-curve sampling method applied to model-based clustering.
\newblock \emph{Journal of Machine Learning Research}, 2\penalty0
  (Feb):\penalty0 397--418, 2002.

\bibitem[Mehrabi et~al.(2019)Mehrabi, Morstatter, Saxena, Lerman, and
  Galstyan]{mehrabi2019survey}
Ninareh Mehrabi, Fred Morstatter, Nripsuta Saxena, Kristina Lerman, and Aram
  Galstyan.
\newblock A survey on bias and fairness in machine learning, 2019.

\bibitem[Moro et~al.(2014)Moro, Cortez, and Rita]{moro2014data}
S{\'e}rgio Moro, Paulo Cortez, and Paulo Rita.
\newblock A data-driven approach to predict the success of bank telemarketing.
\newblock \emph{Decision Support Systems}, 62:\penalty0 22--31, 2014.

\bibitem[Pedregosa et~al.(2011)Pedregosa, Varoquaux, Gramfort, Michel, Thirion,
  Grisel, Blondel, Prettenhofer, Weiss, Dubourg, et~al.]{pedregosa2011scikit}
Fabian Pedregosa, Ga{\"e}l Varoquaux, Alexandre Gramfort, Vincent Michel,
  Bertrand Thirion, Olivier Grisel, Mathieu Blondel, Peter Prettenhofer, Ron
  Weiss, Vincent Dubourg, et~al.
\newblock Scikit-learn: Machine learning in python.
\newblock \emph{Journal of machine learning research}, 12\penalty0
  (Oct):\penalty0 2825--2830, 2011.

\bibitem[Saha et~al.(2020)Saha, Schumann, McElfresh, Dickerson, Mazurek, and
  Tschantz]{Saha20:Measuring}
Debjani Saha, Candice Schumann, Duncan~C. McElfresh, John~P. Dickerson,
  Michelle~L Mazurek, and Michael~Carl Tschantz.
\newblock Measuring non-expert comprehension of machine learning fairness
  metrics.
\newblock In \emph{International Conference on Machine Learning (ICML)}, 2020.

\bibitem[Sweeney(2013)]{sweeney2013discrimination}
Latanya Sweeney.
\newblock Discrimination in online ad delivery.
\newblock \emph{Queue}, 11\penalty0 (3):\penalty0 10--29, 2013.

\bibitem[Verma and Rubin(2018)]{Verma18:Fairness}
Sahil Verma and Julia Rubin.
\newblock Fairness definitions explained.
\newblock In \emph{2018 IEEE/ACM International Workshop on Software Fairness
  (FairWare)}, pages 1--7. IEEE, 2018.

\bibitem[Yeh and Lien(2009)]{yeh2009comparisons}
I-Cheng Yeh and Che-hui Lien.
\newblock The comparisons of data mining techniques for the predictive accuracy
  of probability of default of credit card clients.
\newblock \emph{Expert Systems with Applications}, 36\penalty0 (2):\penalty0
  2473--2480, 2009.

\end{thebibliography}
 
\newpage

\appendix
\section{Omitted Proofs}
\subsection{Proof of Theorem \ref{bera_theorem}}\label{p1}
\begin{proof}
 Let $\instpfc$ a given instance of $\PFC(k,p)$, $SOL_{\PFC}=(S^*_{\PFC},\phi^*_{\PFC})$ the optimal solution of $\instpfc$ and $\OPTPFC$ its corresponding optimal value.
 Also, for $\clust(k,p)$ and for any instance of it, the optimal value is denoted by $\mathrm{OPT}_{\clust}$ and the corresponding solution by $\mathrm{SOL}_{\clust}=(S^*_{\clust},\phi^*_{\clust})$.

The proof closely follows that from \cite{bera2019fair}. First running the color-blind $\alpha$ approximation algorithm results in a set of centers $S$, an assignment $\phi$, and a solution value that is at most $\alpha \mathrm{OPT}_{\clust} \leq \alpha \OPTPFC$. Note that $\mathrm{OPT}_{\clust} \leq \OPTPFC$ since $\PFC(k,p)$ is a more constrained problem than $\clust(k,p)$. Now we establish the following lemma:
\begin{lemma}\label{lemma_for_bera}
$\mathrm{OPT}_{\FAIRAPFC} \leq (\alpha+2) \OPT_{\PFC}$
\end{lemma}
\begin{proof}
The lemma is established by finding the instance satisfying the inequality. Let $\phi'(\point)=\argmin_{i \in S} d(i,\phi^*_{\PFC}(\point))$, i.e. an assignment that routes the vertices from the optimal center to the nearest center in color-blind solution $S$. For any point $\point$ the following holds:
\begin{align*}
    d(\point,\phi'(\point)) & \leq d(\point,\phi^*_{\PFC}(\point)) + d(\phi^*_{\PFC}(\point),\phi'(\point)) \\
    & \leq d(\point,\phi^*_{\PFC}(\point)) +  d(\phi^*_{\PFC}(\point),\phi(\point)) \\
    & \leq d(\point,\phi^*_{\PFC}(\point)) +  d(\point,\phi^*_{\PFC}(\point)) + d(\point,\phi(\point)) \\
    & = 2d(\point,\phi^*_{\PFC}(\point)) + d(\point,\phi(\point))
\end{align*}
stacking the distance values in the vectors $\vec{d}(\point,\phi'(\point)),\vec{d}(\point,\phi^*_{\PFC}(\point)),\text{and } \vec{d}(\point,\phi(\point))$. By the virtue of the fact that $\big(\sum_{\point \in \Points} x^p(\point)\big)^{1/p}$ is the $\ell_p$-norm of the associated vector $\vec{x}$ and since each entry in $\vec{d}(\point,\phi'(\point))$ is non-negative, the triangular inequality for norms implies:
\begin{align*}
    & \big(\sum_{\point \in \Points} d^p(\point,\phi'(\point))\big)^{1/p} \leq 2\big(\sum_{\point \in \Points} d^p(\point,\phi^*_{\PFC}(\point))\big)^{1/p} \\
    & + \big(\sum_{\point \in \Points} d^p(\point,\phi(\point))\big)^{1/p}
\end{align*}
It remains to show that $\phi'$ satisfies the fairness constraints \ref{pfc_eq_2}, for any color $\pcolor$ and any center $i$ in $S$, denote $N(i)=\{j \in S^*_{\PFC}| \argmin_{i' \in S} d(i',j)=i\}$, then we have:
\begin{align*}
    \frac{\sum_{\point \in {\phi'}^{-1}(i)} p^{\pcolor}_{\point}}{|{\phi'}^{-1}(i)|} = \frac{ \sum_{j \in N(i)} \Big( \sum_{\point \in {\phi^*}^{-1}_{\PFC}(j)}  p^{\pcolor}_{\point}  \Big) }{\sum_{j \in N(i)} |{\phi^*}^{-1}_{\PFC}(j)|}
\end{align*}
It follows by algebra and the lower and upper fairness constrain bounds satisfied by $\phi^*_{\PFC}$:
\begin{align*}\label{proprtion_bound_eq}
 l_{\pcolor} & \leq  \min_{j \in N(i)} \frac{  \Big( \sum_{\point \in {\phi^*}^{-1}_{\PFC}(j)}  p^{\pcolor}_{\point}  \Big) }{ |{\phi^*}^{-1}_{\PFC}(j)|} \\ 
& \leq \frac{ \sum_{j \in N(i)} \Big( \sum_{\point \in {\phi^*}^{-1}_{\PFC}(j)}  p^{\pcolor}_{\point}  \Big) }{\sum_{j \in N(i)} |{\phi^*}^{-1}_{\PFC}(j)|}  \\
& \leq \max_{j \in N(i)} \frac{  \Big( \sum_{\point \in {\phi^*}^{-1}_{\PFC}(j)}  p^{\pcolor}_{\point}  \Big) }{ |{\phi^*}^{-1}_{\PFC}(j)|} \\ 
& \leq u_{\pcolor}
\end{align*}
This shows that there exists an instance for $\FAIRAPFC$ that both satisfies the fairness constraints and has cost $\leq 2 \OPTPFC + \alpha \mathrm{OPT}_{\clust} \leq  (\alpha+2) \OPTPFC$. 
\end{proof}
Now combining the fact that we have an $\alpha$ approximation ratio for the color-blind problem, along with an algorithm that achieves a $\gamma$ violation to $\FAIRAPFC$ with a value equal to the optimal value for $\FAIRAPFC$, the proof for theorem \ref{bera_theorem} is complete.
 \end{proof}

\subsection{General Theorem for Lower Bounded Deterministic Fair Clustering}\label{bera_lb_proof}
Before stating the theorem and proof, we introduce some definitions. Let $\FAIRAPFCLB$ denote the fair assignment problem with lower bounded cluster sizes. Specifically, in $\FAIRAPFCLB(S,p,L)$ we are given a set of clusters $S$ and we seek to find an assignment $\phi:\Points\rightarrow S$ so that the fairness constraints \ref{dfclb_eq_2} are satisfied, in addition to constraint \ref{dfclb_eq_3} for lower bounding the cluster size by at least $L$. 

Note that although we care about the deterministic case, the statement and proof hold for the probabilistic case. Since the deterministic case is a special case of the probabilistic, the proof follows for the deterministic case as well. 

\begin{theorem}\label{bera_theorem_lb}
Given an $\alpha$ approximation algorithm for the color blind clustering problem $\mathrm{Cluster}(k,p)$ and a $\gamma$ violating algorithm for the fair assignment problem with lower bounded cluster sizes $\FAIRAPFCLB(S,p,L)$, a solution with approximation ratio $\alpha+2$ and violation at most $\gamma$ can be achieved for the deterministic fair clustering problem with lower bounded cluster size $\DFCLB(k,p)$ in time that is fixed parameter tractable $O(2^k \text{poly}(n))$. 
\end{theorem}
\begin{proof}
First running the color-blind $\alpha$ approximation algorithm results in a set of centers $S$, an assignment $\phi$, and a solution value that is at most $\alpha \mathrm{OPT}_{\clust} \leq \alpha \OPT_{\DFCLB}$.

Now we establish the equivalent to lemma \ref{lemma_for_bera} for this problem:
\begin{lemma}\label{lb_appendix}
For the fair assignment problem with lower bounded cluster sizes $\FAIRAPFCLB$, we can obtain a solution of cost at most $(\alpha+2) \OPT_{\DFCLB}$ in fixed-parameter tractable time $O(2^k \text{poly(n)})$. 
\end{lemma}
\begin{proof}
The proof is very similar to the proof for lemma \ref{lemma_for_bera}. Letting $SOL^*_{\DFCLB}=(S^*_{\DFCLB},\phi^*_{\DFCLB})$ denote the optimal solution to $\DFCLB$ with optimal value $\OPT_{\DFCLB}$. Similarly, define the assignment $\phi'(\point)=\argmin_{i \in S} d(i,\phi^*_{\DFCLB}(\point))$, i.e. an assignment which routs vertices from the optimal center to the closest center in the color-blind solution. By identical arguments to those in the proof of lemma \ref{lemma_for_bera}, it follows that:
\begin{align*}
    & \big(\sum_{\point \in \Points} d^p(\point,\phi'(\point))\big)^{1/p} \leq 2\big(\sum_{\point \in \Points} d^p(\point,\phi^*_{\DFCLB}(\point))\big)^{1/p} \\
    & + \big(\sum_{\point \in \Points} d^p(\point,\phi(\point))\big)^{1/p}\\ 
    & \text{ and that:}\\ 
    & l_{\pcolor} \leq \frac{\sum_{\point \in {\phi'}^{-1}(i)} p^{\pcolor}_{\point}}{|{\phi'}^{-1}(i)|} \leq u_{\pcolor}
\end{align*}
What remains is to show that each cluster is lower bounded by $L$. Here we note that a center in $S$ will either be allocated the vertices of one or more centers in $S^*_{\DFCLB}$ or it would not be allocated any vertices at all. If it is not allocated any vertices, then it is omitted as a center (since no vertices are assigned to it). If vertices for a center or more are routed to it, then it will have a cluster of size $\sum_{j \in N(i)} |{\phi^*}^{-1}_{\DFCLB}(j)| \ge L$. This follows since any center in the optimal solution to $\DFCLB$ must satisfy the lower bound $L$. The issue is that we do not know if a color-blind center is not allocated any vertices and should be omitted. However, we can try all possible close and open combinations for the color-blind centers and solve the $\FAIRAPFCLB$ for each combination. This can be done in $O(2^k \text{poly(n)})$ time (fixed-parameter tractable).  
\end{proof}
Now combining the fact that we have an $\alpha$ approximation ratio for the color-blind problem, along with an algorithm that achieves a $\gamma$ violation to $\FAIRAPFCLB$ with value equal to the optimal value for $\FAIRAPFCLB$, the proof for theorem \ref{bera_lb_proof} is complete.
\end{proof}

\section{Further details on Independent Sampling and Large Cluster Solution}\label{is_details}
Here we introduce more details about independent sampling. In section \ref{is_cb} we discuss the concentration bounds associated with the algorithm.  In section \ref{relax_lc} we show that relaxing the upper and lower bounds might be necessary for the algorithm to have a high probability of success. Finally, in section \ref{lower_bounding_lc} we show that not enforcing a lower bound when solving the deterministic fair instance may lead to invalid solutions. 
\subsection{Independent Sampling and the Resulting Concentration Bounds}\label{is_cb}
We recall the Chernoff bound theorem for the sum of a collection of independent random variables. 
\begin{theorem}\label{chernoff}
    Given a collection of $n$ many binary random variables where $\Pr[X_j=1]=p_{j}$ and $S=\sum_{j=1}^{n} X_j$. Then $\mu=\E[S]=\sum_{j=1}^{n}p_j$ and the following concentration bound holds for $\delta \in (0,1)$:
    \begin{align}
        \Pr(\abs{S-\mu} > \delta \mu) \leq 2 e^{-\mu\delta^2/3}
    \end{align}
\end{theorem}

In the following theorem we show that although we do not know the true joint probability distribution $\distt$, sampling according to the marginal probability  $p^{\pcolor}_{v}$ for each point $\point$ results in the amount of color having the same expectation for any collection of points. But furthermore, the amount of color would have a Chernoff bound for the independently sampled case. 
\begin{theorem}\label{is_t1}
Let $\Pr_{\distt}[X_{1}=x_1,\dots,X_{n}=x_n]$ equal the probability that $(X_1=x_1,\dots,X_n=x_n)$ where $X_i$ is the random variable for the color of vertex $i$ and $x_i \in \Colors$ ($\Colors$ being the set of colors) is a specific value for the realization and the probability is according to the true unknown joint probability distribution $\distt$. Using $X^{\pcolor}_{i}$ for the indicator random variable of color $\pcolor$ for vertex $i$, then for any collection of points $C$, the amount of color $\pcolor$ in the collection is $S^{\pcolor}_{\distt}= \sum_{i \in C} X^{\pcolor}_{i,\distt}$ when sampling according to $\distt$ and it is $S^{\pcolor}_{\disti}= \sum_{i \in C} X^{\pcolor}_{i,\disti}$ when independent sampling is done. We have that:
\begin{itemize}
    \item In general: $\Pr_{\distt}[X_{1}=x_1,\dots,X_{n}=x_n] \neq \Pr_{\disti}[X_{1}=x_1,\dots,X_{n}=x_n]$. 
    \item Expectations agree on the of amount of color: $\E[S^{\pcolor}_{\distt}] = \E[S^{\pcolor}_{\disti}]$. 
    \item The amount of color has a Chernoff bound for the independently sampled case $S^{\pcolor}_{\disti}$. 
\end{itemize}
\end{theorem}
\begin{proof}
The first point follows since we simply don't have the same probability distribution. The second is immediate from the linearity of expectations and the fact that both distributions agree in the marginal probabilities ($\Pr_{\distt}[X_i=\pcolor] = \Pr_{\disti}[X_i=\pcolor]=p^{\pcolor}_i$):
\begin{align*}
    \E[S^{\pcolor}_{\disti}] =& \E \Big[\sum_{i \in C} X^{\pcolor}_{i,\disti} \Big] =  \sum_{i \in C} \E\Big[X^{\pcolor}_{i,\disti} \Big]\\
    & = \sum_{i \in C} p^{\pcolor}_i = \sum_{i \in C} \E\Big[X^{\pcolor}_{i,\distt} \Big] = \E[S^{\pcolor}_{\distt}]
\end{align*}
The last point follows from the fact that $S^{\pcolor}_{\disti}$ is a sum of independent random variables and therefore the Chernoff bound has to hold (\ref{chernoff}). 
\end{proof}
\subsection{Relaxing the Upper and Lower Bounds}\label{relax_lc}
Suppose for an instance $\instpfc$ of probabilistic fair clustering that there exists a color $\pcolor$ for which the the upper and lower proportion bounds are equal, i.e. $l_{\pcolor}=u_{\pcolor}$. Suppose the optimal solution $SOL_{\PFC}=(S^*_{\PFC},\phi^*_{\PFC})$, has a cluster $C_i$ which we assume can be made arbitrarily away than the other points. The Chernoff bound guaranteed by independent sampling would not be useful since the realization has to precisely equal the expectation, not be within a $\delta$ of the expectation. In this case sampling will not result in cluster $C_i$ having a balanced color and therefore the points in $C_i$ would have to merged with other points (if possible, since the entire instance maybe infeasible) to have a cluster with balance equal to $l_{\pcolor}$ and $u_{\pcolor}$ for color $\pcolor$. Since we assumed cluster $C_i$ can be made arbitrarily far away the cost of deterministic instance generated can be arbitrarily worse. 

Note, that we do not really need $l_{\pcolor}=u_{\pcolor}$. Similar arguments can be applied if $l_{\pcolor}\neq u_{\pcolor}$, by assuming the that optimal solution has a cluster $C_i$ (which is arbitrarily far away) whose balance either precisely equals $l_{\pcolor}$ or $u_{\pcolor}$. Simply note that with independent sampling would result in violation to the bounds for cluster $C_i$.  

Therefore, in the worst case relaxing the bounds is necessary to make sure that a valid solution would remain valid w.h.p. in the deterministic instance generated by independent sampling.  

\subsection{Independent Sampling without Lower Bounded Cluster Sizes Could Generate Invalid Solutions}\label{lower_bounding_lc}
\begin{figure}[ht]
\vskip 0.2in
\begin{center}
\centerline{\includegraphics[width=\columnwidth]{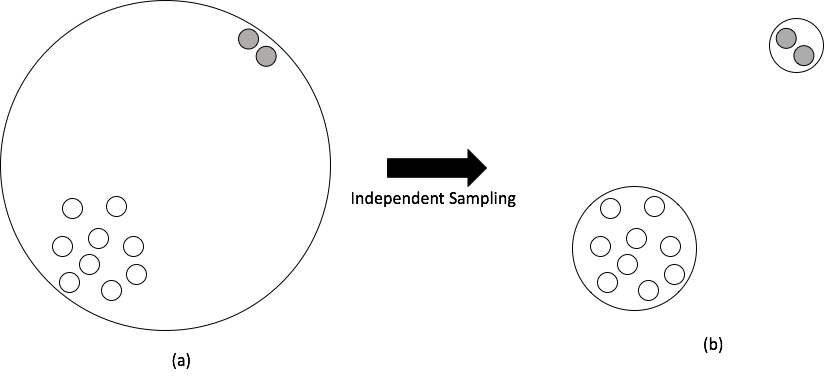}}
\caption{(a): The two outlier points at the top-right have probabilities 0.45 of being white, whereas the rest have probabilities 1. All points are merged together to form a balanced cluster. (b): An instance of same points with the colors resulting from independent sampling. The two outlier points have been merged to form their own cluster.}
\label{indep_sampl_out}
\end{center}
\vskip -0.2in
\end{figure}

To show that enforcing a lower bound on the cluster size is necessary, consider the case shown in figure \ref{indep_sampl_out}:(a) where the outlier points in the top-right have probability $0.45$ of being white, whereas the other points have probability 1 of being white. Let the lower and upper bounds for the white color be $l_{\text{white}}=0.6$ and $u_{\text{white}}=1$, respectively. Since the outlier points don't have the right color balance, they are merged with the other points, although that leads to a higher cost. 

However, independent sampling would result in the outlier points being white with probability $(0.45)(0.45) \simeq 0.2$. This makes the points have the right color balance and therefore the optimal solution for deterministic fair clustering would have these points merged as shown in figure \ref{indep_sampl_out}:(b). However, the cluster for the two outlier points is not a valid cluster for the probabilistic fair clustering instance

Therefore, forcing a lower bound is necessary to make sure that a solution found in deterministic fair clustering instance generated by independent sampling is w.h.p. valid for the probabilistic fair clustering instance.

\section{Example on Forming the Network Flow Graph for the Two-Color (Metric Membership) Case}\label{NF_details}
Suppose we have two centers and 5 vertices and that the LP solution yields the following assignments for center 1: $x_{11}=0.3,x_{12}=0.6,x_{13}=0.7,x_{14}=0,x_{15}=1.0$ and the following assignments for center 2: $x_{21}=0.7,x_{22}=0.4,x_{23}=0.3,x_{24}=1.0,x_{25}=0$. Further let the probability values be: $p_1=0.7,p_2=0.8,p_3=0.4,p_4=0.9,p_5=0.1$. The following explains how the network flow graph is constructed. 

\textbf{Cluster 1:} First we calculate $|C_1|=\ceil{\sum_{j \in \Points} x_{1j}}=\ceil{2.6}=3$, this means the we will have 3 vertices in $C_1$. The collection of vertices having non-zero assignments to center 1 are $\{1,2,3,5\}$, sorting the vertices by a non-increasing order according to their probability we get $\vec{A}_1=[2,1,3,5]$. Now we follow algorithm \ref{alg:nf_const_alg}, this leads to the graph shown in figure \ref{c1_ex}. 
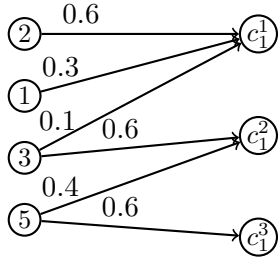
\begin{figure}[ht]
\vskip 0.2in
\begin{center}
\centerline{
\begin{tikzpicture}
[
    xscale=0.45, yscale=0.45,auto,thick,
    st node/.style={
        inner sep=0pt,minimum size=12pt,
        circle, fill=white, draw
  	},
  	v node/.style={
        inner sep=0pt,minimum size=12pt,
        circle, fill=white, draw, font=\small
  	},
  	c node/.style={
        inner sep=0pt,minimum size=14pt,
        circle, fill=white, draw, font=\small
  	},
  	big c node/.style={
        inner sep=0pt,minimum size=14pt,
        circle, fill=white, draw, font=\small
  	},
  	label node/.style={
        font=\small
  	}
]

	
    \node [v node] (vert_2) at (0,0) {$2$};
    \node [v node] (vert_1) [below of=vert_2, yshift=5] {$1$};
    \node [v node] (vert_3) [below of=vert_1, yshift=5] {$3$};
    \node [v node] (vert_5) [below of=vert_3, yshift=5] {$5$};

    \node [c node] (c11) [right of=vert_2, xshift=60] {$c^1_1$};
    \node [c node] (c21) [below of=c11, yshift=-10] {$c^2_1$};
    \node [c node] (c31) [below of=c21, yshift=-10] {$c^3_1$};

    

    \draw[->] (vert_2) -- (c11) node[pos=0.2, above] {$0.6$};
    \draw[->] (vert_1) -- (c11) node[pos=0.1, above] {$0.3$};
    \draw[->] (vert_3) -- (c11) node[pos=0.09, above] {$0.1$};

    \draw[->] (vert_3) -- (c21) node[pos=0.4, above] {$0.6$};
    \draw[->] (vert_5) -- (c21) node[pos=0.1, above] {$0.4$};
    
    \draw[->] (vert_5) -- (c31) node[pos=0.4, above] {$0.6$};


\end{tikzpicture}
}
\caption{Graph constructed in cluster 1. For clarity, we write above each edge the assignment is "sends" to the vertex in $C_1$. Notice how each vertex in $C_1$ receives a total assignment of 1, except for the last vertex $c^3_1$.}
\label{c1_ex}
\end{center}
\vskip -0.2in
\end{figure}

\textbf{Cluster 2:} We follow the same procedure for cluster 2. First we calculate $|C_2|=\ceil{\sum_{j \in \Points} x_{1j}}=\ceil{2.4}=3$, this means the we will have 3 vertices in $C_2$. The collection of vertices having non-zero assignments to center 2 are $\{1,2,3,4\}$, sorting the vertices by a non-increasing order according to their probability we get $\vec{A}_2=[4,2,1,3]$. Now we follow algorithm \ref{alg:nf_const_alg}, this leads to the graph shown in figure \ref{c2_ex}
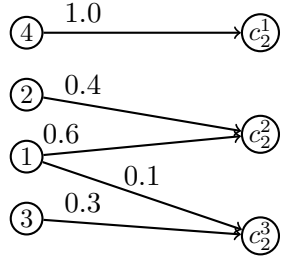
\begin{figure}[ht]
\vskip 0.2in
\begin{center}
\centerline{
\begin{tikzpicture}
[
    xscale=0.45, yscale=0.45,auto,thick,
    st node/.style={
        inner sep=0pt,minimum size=12pt,
        circle, fill=white, draw
  	},
  	v node/.style={
        inner sep=0pt,minimum size=12pt,
        circle, fill=white, draw, font=\small
  	},
  	c node/.style={
        inner sep=0pt,minimum size=14pt,
        circle, fill=white, draw, font=\small
  	},
  	big c node/.style={
        inner sep=0pt,minimum size=14pt,
        circle, fill=white, draw, font=\small
  	},
  	label node/.style={
        font=\small
  	}
]

	
    \node [v node] (vert_4) at (0,0) {$4$};
    \node [v node] (vert_2) [below of=vert_4, yshift=5] {$2$};
    \node [v node] (vert_1) [below of=vert_2, yshift=5] {$1$};
    \node [v node] (vert_3) [below of=vert_1, yshift=5] {$3$};

    \node [c node] (c12) [right of=vert_4, xshift=60] {$c^1_2$};
    \node [c node] (c22) [below of=c12, yshift=-10] {$c^2_2$};
    \node [c node] (c32) [below of=c22, yshift=-10] {$c^3_2$};

    

    \draw[->] (vert_4) -- (c12) node[pos=0.2, above] {$1.0$};
    \draw[->] (vert_2) -- (c22) node[pos=0.2, above] {$0.4$};
    
    \draw[->] (vert_1) -- (c22) node[pos=0.09, above] {$0.6$};
    \draw[->] (vert_1) -- (c32) node[pos=0.5, above] {$0.1$};

    \draw[->] (vert_3) -- (c32) node[pos=0.2, above] {$0.3$};


\end{tikzpicture}
}
\caption{Graph constructed in cluster 2. For clarity, we write above each edge the assignment is "sends" to the vertex in $C_2$. Notice how each vertex in $C_2$ receives a total assignment of 1, except for the last vertex $c^3_2$.}
\label{c2_ex}
\end{center}
\vskip -0.2in
\end{figure}
Now we construct the entire graph by connecting the edges from each vertex in $C_1$ to the vertex for center 1 and each vertex in $C_2$ to the vertex for center 2. Finally, we connect the vertices for 1 and 2 to the vertex $t$. This leads to the graph in figure \ref{entire_graph_ex}. Note that the edge weights showing the sent assignment are not put as they have no significance once the graph is constructed. 

The entire graph is constructed by the union of both subgraphs for clusters 1 and 2, but without repeating the vertices of $\Points$. Further, we drop the edge weights which designated the amount of LP assignment sent, as it has no affect on the following steps. Finally, the vertices of both $C_1$ and $C_2$ are connected to their centers $1$ and $2$ in $S$, respectively, and the centers themsevles are connected to vertex $t$. Figure \ref{entire_graph_ex} shows the final constructed graph. 

For the case of metric membership the procedure is unaltered, but instead of sorting according to the probability value $p_{\point}$ for a vertex, we sort according to the value $r_{\point}$. 
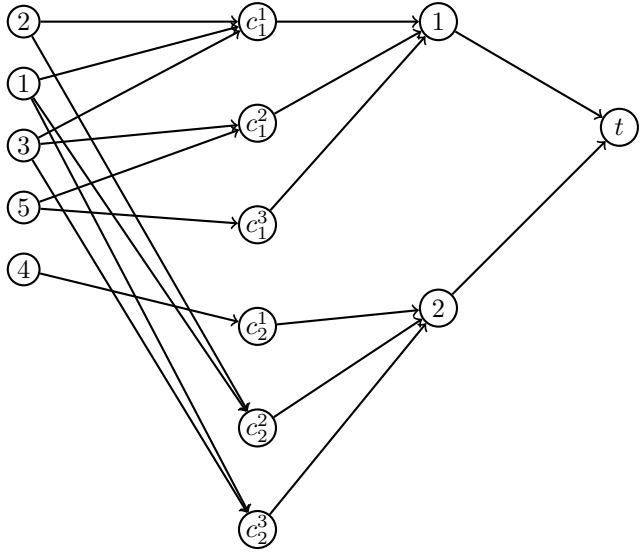
\begin{figure}[ht]
\vskip 0.2in
\begin{center}
\centerline{
\begin{tikzpicture}
[
    xscale=0.45, yscale=0.45,auto,thick,
    st node/.style={
        inner sep=0pt,minimum size=12pt,
        circle, fill=white, draw
  	},
  	v node/.style={
        inner sep=0pt,minimum size=12pt,
        circle, fill=white, draw, font=\small
  	},
  	c node/.style={
        inner sep=0pt,minimum size=14pt,
        circle, fill=white, draw, font=\small
  	},
  	big c node/.style={
        inner sep=0pt,minimum size=14pt,
        circle, fill=white, draw, font=\small
  	},
  	label node/.style={
        font=\small
  	}
]

	
    \node [v node] (vert_2) at (0,0) {$2$};
    \node [v node] (vert_1) [below of=vert_2, yshift=5] {$1$};
    \node [v node] (vert_3) [below of=vert_1, yshift=5] {$3$};
    \node [v node] (vert_5) [below of=vert_3, yshift=5] {$5$};

    \node [v node] (vert_4) [below of=vert_5, yshift=5] {$4$};


    \node [c node] (c11) [right of=vert_2, xshift=60] {$c^1_1$};
    \node [c node] (c21) [below of=c11, yshift=-10] {$c^2_1$};
    \node [c node] (c31) [below of=c21, yshift=-10] {$c^3_1$};
    
    \node [c node] (c12) [below of=c31, yshift=-10] {$c^1_2$};
    \node [c node] (c22) [below of=c12, yshift=-10] {$c^2_2$};
    \node [c node] (c32) [below of=c22, yshift=-10] {$c^3_2$};
    
    \node [c node] (s1) [right of=c11, xshift=40] {$1$};
    \node [c node] (s2) [below of=s1, yshift=-80] {$2$};
    \node [c node] (t) [right of=s1, xshift=40, yshift=-40] {$t$};

    

    \draw[->] (vert_2) -- (c11) node[pos=0.2, above]{};
    \draw[->] (vert_1) -- (c11) node[pos=0.1, above]{};
    \draw[->] (vert_3) -- (c11) node[pos=0.09, above]{};

    \draw[->] (vert_3) -- (c21) node[pos=0.4, above]{};
    \draw[->] (vert_5) -- (c21) node[pos=0.1, above]{};
    
    \draw[->] (vert_5) -- (c31) node[pos=0.4, above]{};

    \draw[->] (vert_4) -- (c12) node[pos=0.2, above]{};
    \draw[->] (vert_2) -- (c22) node[pos=0.2, above]{};
    
    \draw[->] (vert_1) -- (c22) node[pos=0.09, above]{};
    \draw[->] (vert_1) -- (c32) node[pos=0.3, above]{};

    \draw[->] (vert_3) -- (c32) node[pos=0.2, above]{};
    
    \draw[->] (c11) -- (s1) node[pos=0.2, above]{};
    \draw[->] (c21) -- (s1) node[pos=0.2, above]{};
    \draw[->] (c31) -- (s1) node[pos=0.2, above]{};

    \draw[->] (c12) -- (s2) node[pos=0.2, above]{};
    \draw[->] (c22) -- (s2) node[pos=0.2, above]{};
    \draw[->] (c32) -- (s2) node[pos=0.2, above]{};
    
    \draw[->] (s1) -- (t) node[pos=0.2, above]{};
    \draw[->] (s2) -- (t) node[pos=0.2, above]{};


\end{tikzpicture}
}
\caption{Diagram for the final network flow graph.}
\label{entire_graph_ex}
\end{center}
\vskip -0.2in
\end{figure}

\section{Further details on solving the lower bounded fair clustering problem}\label{dfc_lb_appendix}
The solution for the lower bounded deterministic fair clustering problem, follows a similar two step solution framework. Step (1) is unchanged and simply amounts to running a color-blind approximation algorithm with ratio $\alpha$. Step (2) sets up an LP similar to that in section \ref{step2}. The constraints in \ref{LPf1} still remain but with deterministic (not probabilistic) color assignments, further a new constraint lower bounding the cluster size is added. Specifically, we have the following LP:
\begin{alignat}{3}
 \min & \sum_{j \in \Points, i \in S} d^p(i,j) x_{ij}~~~\text{ s.t. }&  \nonumber\\
&l_{\pcolor} \sum_{j \in \Points} x_{ij} \leq \sum_{j \in \Points: \chi(j)=\pcolor} x_{ij}~~&, \forall i \in S, \forall \pcolor \in \Colors \label{LPf1LB} \\
&\sum_{j \in \Points: \chi(j)=\pcolor} x_{ij} \leq u_{\pcolor} \sum_{j \in \Points}x_{ij}~~&, \forall i \in S, \forall \pcolor \in \Colors \label{LPf2LB} \\
&\sum_{j \in \Points} x_{ij} \ge L ~~&, \forall i \in S \label{LB_const}\\
&\sum_{j \in \Points} x_{ij} =1~~&, \forall j \in \Points  \nonumber \\ 
& 0 \leq x_{ij} \leq 1~~&, \forall i \in S ,  \forall j \in \Points \nonumber  
\end{alignat}
Constraints \ref{LPf1LB} and \ref{LPf2LB} are the deterministic counterparts to constraints \ref{LPf1}, respectively. Constraint \ref{LB_const} is introduced to lower bound the cluster size. The issue is that a color-blind center may be closed (assigned no vertices) in the optimal solution, yet constraint \ref{LB_const} forces it to have at least $L$ many points. Therefore the way to fix this is to try all possible combinations of closing and opening the centers which is a total of $2^k$ possibilities. This makes the run time  fixed-parameter tractable time, i.e. $O(2^k \text{poly}(n))$. The resulting solution will have an approximation ratio of $\alpha+2$ (see \ref{bera_lb_proof}). 

What remains is to round the solution. We apply the network flow rounding from \cite{bercea2018cost} (specifically section 2.2 in \cite{bercea2018cost}). This results in a violation of at most 1 in the cluster size and a violation of at most 1 per color in any give cluster (lemma 8 in \cite{bercea2018cost}).

\section{Dependent Rounding for Multiple Colors under a Large Cluster Assumption}\label{tc_mm_sec}
Here we discuss a dependent rounding based solution for the $k$-center problem under the large cluster assumption \ref{lc_assumption}. First we start with a brief review/introduction of dependent rounding.  
\subsection{Brief Summary of Dependent Rounding}
Here we summarize the properties of dependent rounding, see \cite{gandhi2006dependent} for full details. Given a bipartite graph $(G=(A,B),E)$ each edge $(i,j) \in E$ has a value $0 \leq x_{ij} \leq 1$ which will be rounded to $X_{ij} \in \{0,1\}$. Further for every vertex $v \in A \cup B$ define the fractional degree as $d_{v}=\sum_{u:(v,u)\in E} x_{vu}$ and the integral degree as $D_{v}=\sum_{u:(v,u)\in E} X_{vu}$. Dependent rounding satisfies the following properties:
\begin{enumerate}
    \item \label{drp1} $\Pr[X_{ij}=1]=x_{ij}$. 
    \item \label{drp2} $\forall v \in A \cup B: D_{v} \in \{\floor{d_{v}},\ceil{d_{v}}\} $
    \item \label{drp3} $\forall v \in A \cup B$, let $E_{v}$ denote any subset of edges incident on $v$, then $\Pr[\bigwedge_{e_v \in E_{v}} X_{e_{v}}=b] \leq \Pi_{e_v \in E_{v}} \Pr[X_{e_{v}}=b]$ where $b \in \{0,1\}$.  
\end{enumerate}
We note that property \ref{drp3} implies the following theorem about the variables $X_{ij}$ (see theorem 3.1 in \cite{gandhi2006dependent}):
\begin{theorem}\label{dr_high_concentration}
Let $a_1,\dots,a_t$ be reals in $[0,1]$, and $X_1,\dots,X_t$ be random variables taking values in $\{0,1\}$, and $\E[\sum_{i}a_i X_i]=\mu$. If $\Pr[\bigwedge_{i \in S} X_{i}=b] \leq \Pi_{i \in S} \Pr[X_{i}=b]$ where $S$ is any subset of indices from $\{1,\dots,t\}$ and $b \in \{0,1\}$, then we have for $\delta \in (0,1)$:
\begin{align*}
    \Pr\Big[\abs{\sum_{i}a_i X_i -\mu} \ge \delta\mu\Big] \leq 2 e^{-\mu \delta^2/3}  
\end{align*}
\end{theorem}

\subsection{Multiple Color Large Cluster solution using Dependent Rounding}
For the multiple color $k$-center problem satisfying assumption \ref{lc_assumption}. Form the following bipartite graph $(G=(A,B),E)$, $A$ has all vertices of of $\Points$ , and $B$ has all of the vertices of $S$ (the cluster centers). Given fractional assignments $x_{ij}$ that represent the weight of the edge, $\forall (i,j) \in E$. If $x_{ij}$ is the optimal solution to the lower bounded probabilistic fair assignment problem (theorem \ref{bera_lb_proof}), then applying dependent rounding leads to the following theorem:
\begin{theorem}\label{dr_mc}
Under assumption \ref{lc_assumption}, the integer solution resulting from dependent rounding for the multi-color probabilistic $k$-center problem has: (1) An approximation ratio of $\alpha +2$. (2) For any color $h_{\ell}$ and any cluster $i \in S$, the amount of color $S^{h_{\ell}}_{C_i}=\sum_{j \in \Points} p^{h_{\ell}}_j X_{ij}$ is concentrated around the LP assigned color $\sum_{j \in \Points} p^{h_{\ell}}_j x_{ij}$.   
\end{theorem}
\begin{proof}
For (1): Note that the approximation ratio before applying dependent rounding is $\alpha+2$. By property \ref{drp1} of dependent rounding if $x_{ij}=0$, then $\Pr[X_{ij}=1]=0$ and therefore a point will not be assigned to a center it was not already assigned to by the LP. 

For (2): Again by property \ref{drp1} of dependent rounding $\E_{DR}[X_{ij}]= (1) x_{ij} + 0=x_{ij}$ where $\E_{DR}$ refers to the expectation with respect to the randomness of dependent rounding, therefore for any cluster $i$ the expected amount of color equals the amount of color assigned by the LP, i.e. $\E_{DR}[S^{h_{\ell}}_{C_i}]=\E_{DR}[\sum_{j \in \Points} p^{h_{\ell}}_j X_{ij}]=\sum_{j \in \Points} p^{h_{\ell}}_j \E_{DR}[X_{ij}]=\sum_{j \in \Points} p^{h_{\ell}}_j x_{ij}$. It follows by property \ref{drp3} of dependent rounding and theorem \ref{dr_high_concentration} that $S^{h_{\ell}}_{C_i}$ is highly concentrated around $\E_{DR}[S^{h_{\ell}}_{C_i}]$. Specifically :
\begin{align*}
        \Pr\Big[\abs{S^{h_{\ell}}_{C_i} - \E_{DR}[S^{h_{\ell}}_{C_i}]} \ge \delta \E_{DR}[S^{h_{\ell}}_{C_i}]\Big] & \leq 2  e^{-\E_{DR}[S^{h_{\ell}}_{C_i}] \delta^2/3} 
\end{align*}
Similar to the proof of \ref{whp_sample}, the probability of failure can be upper bounded by:
\begin{align*}
    &\Pr\Big(\Bigl\{\exists  i \in \{1,\dots,k\}, \pcolor \in \Colors | \lvert S^{\pcolor}_{C_i} -\E[S^{\pcolor}_{C_i}] \rvert >  \delta \E[S^{\pcolor}_{C_i}]\Bigl\}\Big)   \\
    & \leq 2k \lvert\Colors\rvert \exp(- \frac{\delta^2}{3} L l_{\min}) \leq 2 \frac{n}{L} \lvert\Colors\rvert\exp(- \frac{\delta^2}{3} L l_{\min})\\
    & \leq 2 \lvert\Colors\rvert n^{1-r} \exp(- \frac{\delta^2}{3}  l_{\min} n^{r})
\end{align*}
Therefore w.h.p the returned integral solution will be concentrated around the LP color assignments which are fair. 
\end{proof}
Note however, that obtaining the optimal fractional solution $x_{ij}$ takes $O(2^k \text{poly}(n))$ time.

\section{Further Experimental Details and Results}\label{further_exps}

\subsection{Further Details about the Datasets and the Experimental Setup}
For each dataset, the numeric features are used as coordinates and the distance between points is equal to Euclidean distance. The numeric features are normalized prior to clustering. 

For metric membership in the $\adult$ dataset, age is not used as a coordinate despite the fact that it is numeric since it is the fairness attribute. Similarly, for the $\credit$ dataset, credit is not used as a coordinate.  

When solving the min-cost flow problem, distances are first multiplied by a large number (1000) and then rounded to integer values. After obtaining the solution for the flow problem, the cost is calculated with the original distance values (which have not been rounded) to verify that the cost is not worse.  

Although run-time is not a main concern in this paper. We find that we can solve large instances containing 100,000 points for the $k$-means with 5 clusters in less than 4 minutes using our commodity hardware.  

\subsection{Further Experiments}
Here we verify the performance of our algorithm on the $k$-center and the $k$-median objectives.  All datasets have been sub-sampled to 1,000 data points. For the two color probabilistic case, throughout we set $\pacc=0.9$ (see section \ref{sec:experiments-two-color} for the definition of $\pacc$). 
\subsubsection{$k$-center}
As can be seen from figure \ref{kcenterprob} our violations are indeed less than 1 matching the theoretical guarantee. Similarly, for metric membership the normalized violation is less than 1 as well, see figure \ref{kcenterage}. 
\begin{figure}[ht]
\vskip 0.2in
\begin{center}
\centerline{\includegraphics[width=\columnwidth]{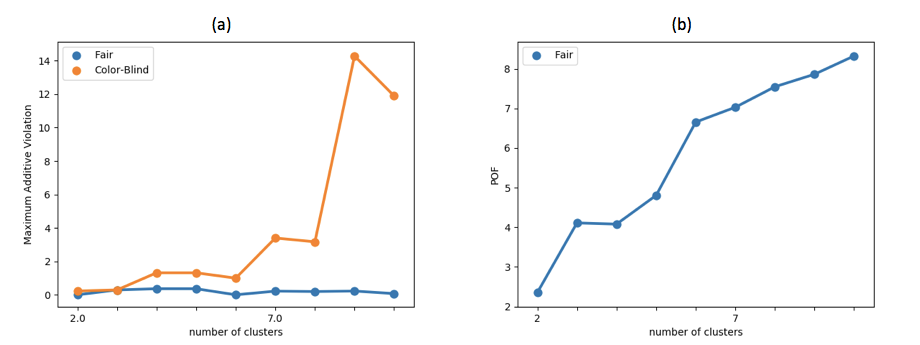}}
\caption{$k$-center for the two color probabilistic case using the $\bank$ dataset. (a): number of clusters vs maximum violation, (b): number of clusters vs POF.}
\label{kcenterprob}
\end{center}
\vskip -0.2in
\end{figure}

\begin{figure}[ht]
\vskip 0.2in
\begin{center}
\centerline{\includegraphics[width=\columnwidth]{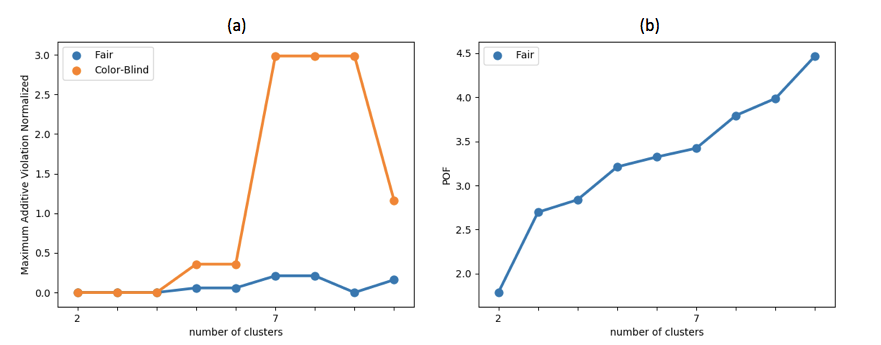}}
\caption{$k$-center for the metric membership problem using the $\adult$ dataset (metric membership over age). (a): number of clusters vs normalized maximum violation, (b): number of clusters vs POF.}
\label{kcenterage}
\end{center}
\vskip -0.2in
\end{figure}

\subsubsection{$k$-median}
Similar observations apply to the $k$-median problems. That is, our algorithm indeed leads to small violations not exceeding 1 in keeping with the theory. See figure \ref{kmedianprob} for the two color probabilistic case and figure \ref{kmedrage} for the metric membership case. 
\begin{figure}[ht]
\vskip 0.2in
\begin{center}
\centerline{\includegraphics[width=\columnwidth]{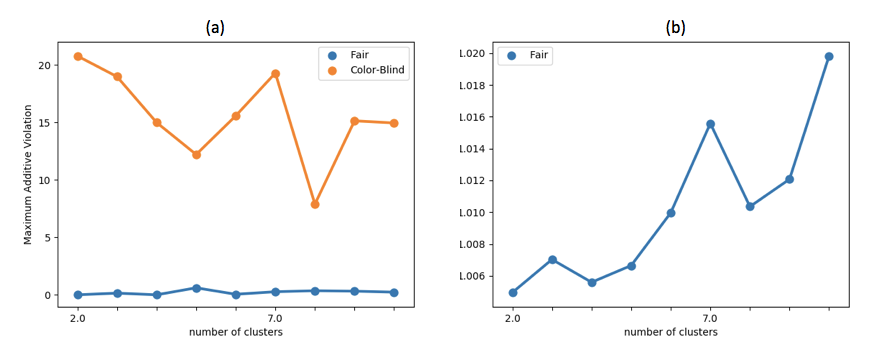}}
\caption{$k$-median for the two color probabilistic case using the $\bank$ dataset. (a): number of clusters vs maximum violation, (b): number of clusters vs POF.}
\label{kmedianprob}
\end{center}
\vskip -0.2in
\end{figure}

\begin{figure}[ht]
\vskip 0.2in
\begin{center}
\centerline{\includegraphics[width=\columnwidth]{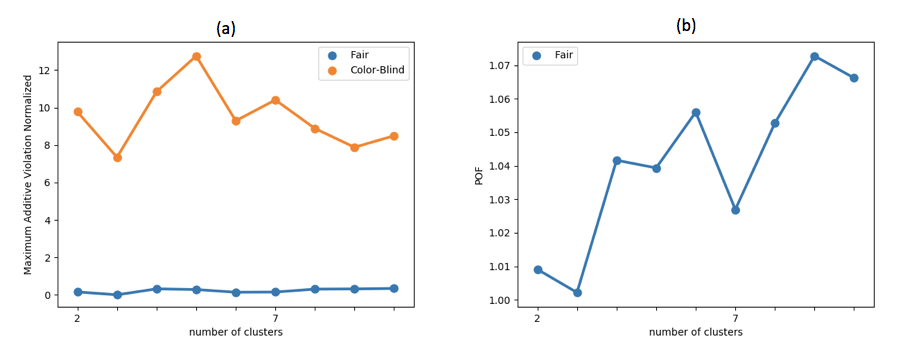}}
\caption{$k$-median for the metric membership problem using the $\credit$ dataset (metric membership over credit) (a): number of clusters vs normalized maximum violation, (b): number of clusters vs POF.}
\label{kmedrage}
\end{center}
\vskip -0.2in
\end{figure}

\end{document}